\documentclass[american,a4paper,runningheads,envcountsect]{llncs}
\pdfoutput=1
\usepackage{babel}
\usepackage[utf8]{inputenc}
\usepackage[T1]{fontenc}

\usepackage{amsmath}
\usepackage{amsthm}
\usepackage{amssymb}
\usepackage{mathtools}
\usepackage{mathrsfs}

\usepackage{dsfont} 

\usepackage{subcaption}
\usepackage[ruled,vlined,linesnumbered]{algorithm2e}
\usepackage{graphicx}
\usepackage{float} 

\usepackage{tikz}
\usetikzlibrary{shapes,arrows,positioning,calc} 

\usepackage[shrink=25,babel,kerning=true]{microtype}
\usepackage{csquotes}
\usepackage{booktabs}
\usepackage{enumitem}
\usepackage[subtle]{savetrees}
\usepackage{todonotes}
\presetkeys{todonotes}{color=blue!5}{}

\usepackage[backend=biber,style=numeric-comp,doi=false,isbn=false,%
url=false]{biblatex}
\usepackage[colorlinks,citecolor=blue,linkcolor=blue,urlcolor=black,linktocpage]{hyperref}
\usepackage[all]{hypcap}

\usepackage{cleveref}
\let\cref\Cref
\usepackage{fca_mod}

\newtheorem{thm}{Theorem}[section]
\newtheorem{cor}[thm]{Corollary}
\newtheorem{lem}[thm]{Lemma}

\newtheorem{fact}[thm]{Fact}

\theoremstyle{definition}

\newtheorem{defn}[thm]{Definition}
\newtheorem{notation}[thm]{Notation}
\newtheorem{strategy}[thm]{Strategy}
\theoremstyle{remark}
\newtheorem{exmpl}[thm]{Example}
\newtheorem{rem}[thm]{Remark}

\crefname{defn}{definition}{definitions}
\Crefname{defn}{Definition}{Definitions}
\crefname{lem}{lemma}{lemmas}
\Crefname{lem}{lemma}{Lemmas}

\newcommand\blfootnote[1]{%
  \begingroup
  \renewcommand\thefootnote{}\footnote{#1}%
  \addtocounter{footnote}{-1}%
  \endgroup
}

\newcommand{\tand}{\text{ and }}

\newcommand{\tif}{\text{ if }}

\newcommand{\totherwise}{\text{ otherwise}}
\newcommand{\ttrue}{\text{true}}
\newcommand{\tfalse}{\text{false}}
\newcommand{\tunknown}{\text{unknown}}

\newcommand{\x}{\times}
\renewcommand{\o}{o}
\newcommand{\q}{?}
\newcommand{\values}{\ensuremath{\{\x,\o,\q\}}}
\renewcommand{\GMI}[1][]{(G,M,\relI_{#1})}
\newcommand{\GMWI}[1][]{(G,M,\values,\relI_{#1})}
\newcommand{\GMWII}[1][]{(G_{#1},M,\values,\relI_{#1})}
\renewcommand{\implies}{\rightarrow}
\newcommand{\Implies}{\Rightarrow}

\newcommand{\formulas}[1]{\ensuremath{\mathit{F}(#1)}}
\newcommand{\powerset}[1]{\ensuremath{\mathcal{P}(#1)}}
 
\newcommand{\respectingmodels}[1]{\ensuremath{\operatorname{Resp}(#1)}}

\newcommand{\Impm}{\ensuremath{\operatorname{Imp}_M}}
\newcommand{\Imp}[1][\context]{\operatorname{Imp}(#1)}
\newcommand{\Sat}[1][\context]{\operatorname{Sat}(#1)}
\newcommand{\Conss}[3]{\operatorname{Cons}^{#2}_{#3}(#1)}
\newcommand{\Cons}[1]{\Conss{#1}{}{}}

\newcommand{\expert}[1]{(\context_{#1},\Cons{\mathcal{L}_{#1}})}

\renewcommand{\epsilon}{\varepsilon}
\renewcommand{\phi}{\varphi}

\begin{document}

\title{Interactive Collaborative Exploration using Incomplete
  Contexts}

\author{Maximilian Felde \and  Gerd Stumme}

\date{\today}

\institute{%
  Knowledge and Data Engineering Group\\[0.5ex]
  Research Center for Information System Design\\[0.5ex]
  University of Kassel, Germany\\[0.5ex]
  \email{felde@cs.uni-kassel.de, stumme@cs.uni-kassel.de}}
\maketitle

\blfootnote{Authors are given in alphabetical order.  No priority in
  authorship is implied.}

\begin{abstract}

  A well-known knowledge acquisition method in the field of Formal
  Concept Analysis (FCA) is \emph{attribute exploration}.  It is used
  to reveal dependencies in a set of attributes with help of a domain
  expert.  In most applications no single expert is capable (time- and
  knowledge-wise) of exploring the knowledge domain alone.  However,
  there is up to now no theory that models the interaction of multiple experts
  for the task of attribute exploration with incomplete knowledge.
  To this end, we to develop a theoretical framework that
  allows multiple experts to explore domains together.
  We use a representation of incomplete knowledge as
  three-valued contexts.
  We then adapt the corresponding version of attribute exploration to fit
  the setting of multiple experts. We suggest formalizations for key
  components like expert knowledge, interaction and collaboration
  strategy.  In particular, we define an order that allows to compare
  the results of different exploration strategies on the same task
  with respect to their information completeness.  Furthermore we
  discuss other ways of comparing collaboration strategies and suggest
  avenues for future research.
\end{abstract}


\vspace{3ex}                    %
\section{Introduction}
\label{sec:introduction}
Nowadays information is generated and collected at unimaginable
scales.  Some of it is published online, for example on Wikipedia or
in knowledge bases such as Wikidata or DBpedia.  Collecting
information from a domain is the first step to acquiring knowledge.
Often the next step is to structure the information and extract
conceptual knowledge, a task performed by experts of the domain.  But
even for domains of reasonable size experts normally have incomplete
knowledge and collaboration is necessary to improve the results.

For domains that can be represented as data-tables of objects and
attributes, Formal Concept Analysis (FCA)~\cite{GanterWille1999}
provides the well-known knowledge acquisition method of
\emph{attribute
  exploration}~\cite{ganter1984two,ganter2016conceptual}.  
This method helps an expert to systematically obtain knowledge about the structural
dependencies between attributes.


In this paper, we present a theoretical framework for a cooperative
attribute exploration that allows for a set of experts.
Before diving into technical details, however, we illustrate in the
next three paragraphs by an example the procedure of state-of-the-art
attribute exploration with a single expert.
To this end, imagine that we want to explore dependencies of properties
of the sports disciplines of the Summer Olympics 2020 in Tokyo.
 Some properties of interest could be if a discipline has
individual or team competitions, if the contestants of events are
males, females or mixed, or how many events are held and how often the
discipline was already part of the Olympic
Games\footnote{\label{footnote:Olympic_Games_Sources}The information
  for this example was obtained from \url{https://tokyo2020.org/},
  \url{https://www.olympic.org/tokyo-2020} and
  \url{https://en.wikipedia.org/wiki/Olympic_sports}. The full example
  context can be found in the
  \hyperref[appendix:full_example_context]{appendix}.}.

For the basic version of attribute exploration to work, we need all
properties to be binary attributes.  For this purpose, FCA provides a
technique called \emph{scaling}, cf. \cite[Sec. 1.3
ff.]{GanterWille1999}, to transform each non-binary property to
multiple binary properties.  Here we use the three properties `has at
least 5, 10 or 20 events' for the number of events of a discipline and
the three properties `was part of at least 8, 16 or 24 Olympic Games'
for the number of Olympic Games 
that a discipline was already part of, both of which are examples of
scaling using an ordinal scale.

The basic attribute exploration algorithm developed by
Ganter~\cite{ganter1984two} works as follows: The algorithm
systematically asks questions such as \emph{`Do all disciplines have events
for males and for females?'} and \emph{`Have disciplines that hold at least
ten events been part of at least eight Olympic Games?'}.  
These questions are answered by the expert who conducts the
exploration.
More precisely, the expert confirms a question if it is true or
rejects it with a counterexample if it is false.
Here, the expert would reject the
first question since `Artistic Swimming' is an Olympic discipline that
only has events for females.
The expert would further report all properties that `Artistic
Swimming' has. 
More specifically, that it has no male events, no mixed events and no
individual events but does have team events, and that the total number of events is less
than five and that the discipline was part of at least eight but no
more than fifteen Olympic Games.
The expert would confirm the second question to be a valid implication.
The attribute exploration systematically asks such questions
until all possible questions can either be inferred from the set of
accepted questions or rejected based on the counterexamples given by the
expert.

Burmeister and Holzer have developed an extension of attribute
exploration where the expert may have incomplete knowledge of the
domain,
cf. \cite{burmeister1991merkmalimplikationen,conf/iccs/BurmeisterH00,holzer2001dissertation,holzer2004knowledgeP1,holzer2004knowledgeP2}.
In this setting, the expert can also answer questions with ``I don't
know''.  Furthermore, the expert does not need to know the relations of an
object to all attributes when providing a counterexample. This
extension makes attribute exploration more viable in practice. In most
applications, however, no single expert is capable (time- and knowledge-wise) of
exploring a domain on her own. Yet, to the best of our knowledge
there 
presently exists no theory that allows multiple experts with
incomplete knowledge to cooperatively perform attribute exploration.

The purpose of this paper is to 
provide such a theoretical framework, i.e., develop 
the attribute exploration with incomplete knowledge to work with
multiple experts.  To this end, we suggest a possible formalization of
attribute exploration in a collaborative setting with experts that
have incomplete knowledge.  
As a first step, we formalize expert knowledge.
Then, we introduce an order relation that allows us to compare and
combine expert knowledge.
We develop formalizations for expert interaction and
collaboration strategy.
Furthermore, we define an order that enables the comparison of
  results of different exploration strategies on the same task
  with respect to their information completeness.
In order to discuss methods to evaluate and
compare collaboration strategies we use some examples of collaboration
strategies.  These examples also reveal some possible improvements to
our approach.

We mainly build upon the results of Burmeister and Holzer about
representing incomplete knowledge and the corresponding version of
attribute exploration for a single expert.
Note that, in this paper, we neither consider imprecise
knowledge, e.g., the case where an expert is $80\%$ sure that an
object has an attribute or $90\%$ sure an implication is valid, nor
contradictory knowledge, e.g., the case where experts disagree whether
an object has attribute.  


We begin by giving a review of related work in
\cref{sec:related_work}.  Then we recapitulate the existing theory needed
to formulate attribute exploration for a single expert with incomplete
knowledge in \cref{sec:recollection_of_definitions_and_results}.
Afterwards we develop our theory to handle multiple experts with
incomplete knowledge in a collaborative setting in
\cref{sec:experts_and_collaboration}.  Examples for collaboration
strategies are given and used to discuss methods to compare different
strategies.
In \cref{sec:conclusion_and_outlook} we give a conclusion and
recollect some avenues for future research.
Lastly, this paper has a running example where we provide the
\emph{Sports Disciplines of the Summer Olympics 2020} context and an
extensive example of three experts performing a collaborative
exploration with incomplete knowledge.
For better readability, this example is presented in one piece in
an appendix (\cref{sec:appendix}), but is cross-referenced throughout
the whole paper wherever appropriate.

\section{Related Work}
\label{sec:related_work}
Much work has already been done concerning the modelling of
uncertainty and incomplete knowledge in particular, e.g., Bayesian
statistics, modal logics, possibility theory and probabilistic logics.
Of particular interest for us are the three-valued logics
Kleene-Algebras~\cite{Kleene1968-KLEITM-3} and
Kripke-Semantics~\cite{kripke_doi:10.1002/malq.19630090502} as they
have been used to model incomplete knowledge in
FCA~\cite{burmeister1991merkmalimplikationen,Pagliani1997}.

In the realm of FCA, attribute exploration for incomplete knowledge
has been around for about 30 years.  The first attempts to model
incomplete knowledge in the context of formal concept analysis were
made by Burmeister using Kleene-logic in
\cite{burmeister1991merkmalimplikationen} where he already discussed
attribute exploration with incomplete knowledge and strategies to deal
with questions that can not be answered directly.

In later works different approaches were explored.  For example, in
\cite{ganter1999attribute} Ganter models incomplete knowledge using
two formal contexts: One for attributes an object certainly has and
one for attributes an object possibly has.  Another example is
\cite{DBLP:conf/iccs/Obiedkov02} where Obiedkov discussed the
evaluation of propositional formulas in incomplete contexts using a
three-valued modal logic with a `nonsense' value.  In 2016 the book
`Conceptual Exploration'~\cite{ganter2016conceptual} written by Ganter
and Obiedkov was published giving an extensive overview on the many
variations of attribute
exploration.

The paper `On the Treatment of Incomplete Knowledge in Formal Concept
Analysis' \cite{conf/iccs/BurmeisterH00} by Burmeister and Holzer
gives a good overview on how to treat incomplete knowledge in
attribute exploration.  It covers a wide range of topics from an
introduction of incomplete contexts, the definitions of possible and
certain intents and extents and attribute implications in incomplete
contexts to reductions of question-marks, three-valued Kleene-logic
and an algorithm for attribute exploration that allows questions to be
answered with `true', `false' or `unknown'. However,
attribute exploration remains focused on a single expert.

The thesis of Holzer \cite{holzer2001dissertation} and the later
adaptions as papers \cite{holzer2004knowledgeP1,holzer2004knowledgeP2}
contain in-depth results about incomplete contexts, their relationship
to attribute implications and a version of attribute exploration that
allows the expert to answer questions with ``I don't know''. 
As before, attribute exploration remains focused on a single expert.

The publications of Burmeister and Holzer can be considered the most
elaborate for dealing with incomplete knowledge in the realm of FCA.
Hence, we use their results as a foundation for our work on
collaborative exploration with incomplete knowledge.


There exist some publications that deal with ideas of collaboration in the
realm of FCA but not nearly as many as deal with incomplete knowledge.
%
%
A publication specifically addressing collaborative attribute exploration
to help with ontology construction is
\cite{obiedkov2015collaborative} by Obiedkov and Romashkin. Here, some issues arising with
collaborative exploration are identified, e.g., the need to allow for
incomplete examples and for having policies of collaboration that can
deal with conflicting information. However, these issues are merely stated
and not examined in detail. Consequently the task of
further improving the theoretic foundations of attribute exploration
in a collaborative setting is raised.

Recently Hanika and Zumbrägel suggested an approach for collaborative
exploration based on experts for attribute sets
\cite{conf/iccs/HanikaZ18}.  They employ the notion of a consortium of
experts and discuss its ability, i.e., how much of the domain can be
explored given certain experts for attribute sets, and the value of
being able to combine examples.  However, the experts do not directly
talk about objects in the domain (which makes it impossible to merge
partial knowledge about the same object) and are not allowed to answer
``I dont't know''.

Another obstacle of an efficient interactive collaborative attribute
exploration is the sequentiality of asking questions when utilizing
the NextClosure algorithm,
cf. Ganter \cite{ganter1984two,GanterWille1999,ganter2016conceptual}, to generate
questions.  In \cite{conf/iccs/Kriegel16}, Kriegel
modified the NextClosure algorithm to obtain a parallel version of attribute exploration with
all-knowing experts.
This allows multiple questions to be generated at once
and might help to further improve the efficiency of collaborative attribute
exploration with incomplete knowledge.

%
\section{Recollection of known Results}
\label{sec:recollection_of_definitions_and_results}
In this section, we recollect some basic definitions from FCA as
introduced in \cite{Wille82} and recapitulated in
\cite{GanterWille1999} and recollect notions and results from
\cite{holzer2001dissertation,conf/iccs/BurmeisterH00,holzer2004knowledgeP1}
for incomplete contexts and attribute exploration for incomplete
knowledge.
We add some notation to make things more readable in the following
sections and give a few examples to ease understanding of the core
ideas.

\subsection{Formal Context}
Formal contexts are one of the most basic structures in FCA.  Note
that we consider the incidence relation as a function to better fit our
needs later on.
\begin{defn}[formal context, c.f. \cite{Wille82}]
  A (one-valued) \emph{formal context} $\context = \GMI$ consists of a
  set of objects $G$, a set of attributes $M$ and an incidence
  relation $I\subseteq G\times M$ with $(g,m)\in I$ meaning \emph{the
    object $g$ has the attribute
    $m$}.
\end{defn}
There are several interpretations for $(g,m)\not\in I$,
cf. \cite{burmeister1991merkmalimplikationen,conf/iccs/BurmeisterH00},
the standard one being, ``\emph{$g$ does not have the attribute $m$ or it is
  irrelevant, whether $g$ has $m$}''.  In the following we interpret
$(g,m)\not\in I$ as ``\emph{$g$ does not have $m$}'', which is reasonable
when modeling incomplete knowledge.
    
This interpretation can be equivalently modeled by a (two-valued)
\emph{formal context} $\context = \GMI$ that consists of a
set of objects $G$, a set of attributes $M$ and an \emph{incidence function}
$I:G\times M \to \{\x,\o\}$.  The incidence function describes whether
an object $g$ has an attribute $m$:  $I(g,m)=\x$ means ``\emph{$g$ has
  $m$}'' and $I(g,m)=\o$ means ``\emph{$g$ does not have $m$}''.  Clearly
we can use a one-valued formal context to define an equivalent
two-valued formal context and vice versa using
$(g,m)\in I \Leftrightarrow I(g,m)=\x$.
%
In the following we will use these two definitions interchangeably.

A formal context can be represented as table with rows of objects and
columns of attributes.  The table entries signify the relation of
objects and attributes.  It is customary to put an ``$\x$'' to
indicate that an object has an attribute and to either put an ``$\o$''
or leave a table cell empty to indicate that an object does not have
an attribute.  For readability we will leave the cells empty. See
\cref{fig:formal_context_examples} for representations of a formal
context using part of the introductory example about properties of the
sports disciplines of the Summer Olympic Games 2020.

\begin{figure}
  \begin{center}
    \begin{cxt}
      \cxtName{}
      \att{a} \att{b} \att{c} \att{d} \att{e}
      \obj{xxxxx}{Aquatics -- Swimming} \obj{.xxxx}{Badminton}
      \obj{..x.x}{Gymnastics -- Rhythmic}
    \end{cxt}
    \begin{cxt}
      \cxtName{} \cxtNichtKreuz{}
      \att{a} \att{b} \att{c} \att{d} \att{e}
      \obj{xxxxx}{Aquatics -- Swimming}
      \obj{.xxxx}{Badminton}
      \obj{..x.x}{Gymnastics -- Rhythmic}
    \end{cxt}
  \end{center}
%
%
  \caption{Example of two representations of a formal context about
    the sports disciplines of the Summer Olympic Games 2020, see
    \cref{footnote:Olympic_Games_Sources}, with attributes
    a)~\emph{$\geq$~10~events}, b)~\emph{$\geq$~5~events},
    c)~\emph{female~only~events}, d)~\emph{male~only~events}, and
    e)~\emph{part of $\geq$~8 Olympics}.  }
  \label{fig:formal_context_examples}
\end{figure}

\begin{defn}(derivation operators, c.f. \cite{Wille82}) Let
  $\context = \GMI$ be a formal context.  For a set of objects
  $A\subseteq G$ we define the set of attributes common to the objects
  in $A$ by
  \[
    A' \coloneqq \{ m \in M \mid \forall g\in A: I(g,m)=\x \}.
  \]
  Analogously, for a set of attributes $B\subseteq M$ we define the
  set of objects that have all the attributes from $B$ by
  \[
    B' \coloneqq \{ g \in G \mid \forall m\in B: I(g,m)=\x \}.
  \]
\end{defn}
\begin{defn}[formal concept, intent, extent, c.f.
  \cite{Wille82}]
  Let $\context = \GMI$ be a formal context.  A \emph{formal concept}
  of $\context$ is a pair $(A,B)$ with $A\subseteq G$ and
  $B\subseteq M$ such that $A'=B$ and $A=B'$.  We call $A$ the
  \emph{extent} and $B$ the \emph{intent} of the formal concept
  $(A,B)$.
\end{defn}
Note that for any set $A \subseteq G$ the set $A'$ is the intent of a
concept and for any set $B\subseteq M$ the set $B'$ is the extent of a
concept.
\subsection{Incomplete Context}
In the following we introduce \emph{incomplete contexts} as a means to
model partial knowledge.  Note that once again we use an incidence
function.
\begin{defn}[incomplete context, c.f. \cite{conf/iccs/BurmeisterH00}]
  An \emph{incomplete context} is a three-valued context
  $\context = \GMWI$ consisting of a set of objects $G$, a set of
  attributes $M$, a set of values $\values$ and an incidence function
  $I:G\times M \to \values$.  For $g\in G$ and $m\in M$ we say that
  ``\emph{it is known that $g$ has $m$}'' if $I(g,m)=\times$, ``\emph{it is
    known that $g$ does not have $m$}'' if $I(g,m)=\o$ and ``\emph{it is
    not known whether $g$ has $m$}'' if $I(g,m)=\q$.
\end{defn}
Like a formal context, an incomplete context can be represented as a
table of objects and attributes with the entries signifying the
relation.  Here we use ``$\x$'' to indicate that object and attribute
are known to be related, an empty cell to indicate they are known not
to be related and ``$\q$'' to indicate that the relation is not known.
See \cref{fig:incomplete_context_example} for an example of incomplete
knowledge added to the example in
\cref{fig:formal_context_examples}.
\begin{figure}
  \begin{center}
    \begin{cxt}
      \cxtName{} \cxtNichtKreuz{}
      \att{a} \att{b} \att{c} \att{d} \att{e}
      \obj{xxxxx}{Aquatics -- Swimming} \obj{.xxxx}{Badminton}
      \obj{..x.x}{Gymnastics -- Rhythmic} \obj{??xx?}{Taekwondo}
    \end{cxt}
  \end{center}
  \caption{Example of an incomplete contexts with attributes
    a)~\emph{$\geq$~10~events}, b)~\emph{$\geq$~5~events},
    c)~\emph{female~only~events}, d)~\emph{male~only~events}, and
    e)~\emph{part of $\geq$~8 Olympics}.  Imagine someone saw the
    context in \cref{fig:formal_context_examples} and further knew
    that `Taekwondo' is an Olympic discipline but was unsure how many
    events there are and for how long it has been Olympic.}
  \label{fig:incomplete_context_example}
\end{figure}

\begin{notation}[$\relI^{\x}$\!, $\relI^{\o}$, $\relI^{\q}$]
  Let $\context = \GMWI$ be an incomplete context.  To refer to
  certain subsets of $G\times M$ we define:
  \begin{align*}
    \relI^{\x} &\coloneqq \{(g,m)\in G\times M \mid I(g,m)=\x \}\\
    \relI^{\o} &\coloneqq \{(g,m)\in G\times M \mid I(g,m)=\rlap{\:\!$\o$}\phantom{\x} \}\\
    \relI^{\q} &\coloneqq \{(g,m)\in G\times M \mid I(g,m)=\rlap{\:\!$\q$}\phantom{\x} \}
  \end{align*}
    %
\end{notation}
\begin{notation}[$\context|_{A}$]
  The \emph{restriction} of an incomplete context ${\context = \GMWI}$
  \emph{to a subset} $A\subseteq G$ is denoted by
  $\context|_{A} \coloneqq (A,M,\values,I|_{A\times M})$, where
  \[
    I|_{A\times M}\colon A\times M \rightarrow \values \text{ with }
    I|_{A\times M}(g,m) = I(g,m).
  \]
\end{notation}
If an incomplete context $\context = \GMWI$ does not contain any
``$\q$'', i.e., all relations of objects and attributes are known, it
can be identified with a formal context
$\tilde{\context} = (G,M,\relI)$ with $I = \relI^{\x}$.  We also call
such a context \emph{complete} incomplete context.

Any formal context $\context = \GMI$ can also be identified with an
incomplete context where the incidence relation is completely known,
i.e,.  as a context $\context = (G,M,\values,J)$ where
\[
  J(g,m) \coloneqq \begin{cases}
    \x &\text{ if } (g,m) \in I\\
    \o &\text{ if } (g,m) \not\in I
  \end{cases}
\]
Therefore complete incomplete contexts and formal contexts will be
used synonymously, the particular representation will be mentioned if
it is necessary and can not be inferred from the context.

Similar to the case of formal contexts we define derivation operators
for incomplete contexts.  Since it may be unknown if an object has an
attribute we define two operators, one for the relations that are
known and one for the relations that are possible.
\begin{defn}[certain and possible derivation operators,
  c.f. \cite{conf/iccs/BurmeisterH00,holzer2004knowledgeP1}]
  Given an incomplete context $\context = \GMWI$ we define the
  \emph{certain intent} for $A\subseteq G$ by
  \[
    A^\Box \coloneqq \{ m \in M \mid I(g,m)=\x \text{ for all } g\in A
    \}
  \]
  and the \emph{possible intent} by
  \begin{align*}
    A^\Diamond \coloneqq & \{ m \in M \mid I(g,m)\in\{\x,\q\} \text{ for all } g\in A \}\\
    = & \{ m \in M \mid I(g,m)\neq \o \text{ for all } g\in A \}.
  \end{align*}

  For $B\subseteq M$ we define the \emph{certain extent} $B^\Box$ and
  the \emph{possible extent} $B^\Diamond$ in the same way.  For
  $g\in G$ and $m\in M$ we use the abbreviations $g^\Box$,
  $g^\Diamond$, $m^\Box$ and $m^\Diamond$.
\end{defn}

\begin{exmpl}
  Recall the incomplete context from
  \cref{fig:incomplete_context_example}.  Let us take a look at the
  \emph{possible intent} and the \emph{certain intent} of
  $A=\{\text{Taekwondo, Badminton}\}$.  This means we look at the
  set of attributes that both \emph{Taekwondo} and \emph{Badminton}
  certainly have and at the set of attributes that they possibly have.
  Here, the \emph{possible intent}
  $ \{\text{Taekwondo, Badminton}\}^\Diamond$ is
  $ \{\text{b,c,d,e}\}$, whereas the \emph{certain intent}
  $ \{\text{Taekwondo, Badminton}\}^\Box$ is $ \{\text{c,d}\}$.
\end{exmpl}

\begin{rem}
  In the case of a formal context $\context = \GMI$ (or a complete
  incomplete context $\context=\GMWI$, i.e., with
  $\relI^\q=\emptyset$) the certain and possible intent and extent are
  the same and are equivalent to the usual intent and extent for
  formal contexts, i.e., for $A\subseteq G$ and $B\subseteq M$ we have
    $$A'= A^\Box = A^\Diamond$$
    $$B'= B^\Box = B^\Diamond$$
  \end{rem}

\subsection{Order on Incomplete Contexts}
The values $\values$ can be ordered in at least two ways that make
sense semantically: We can order them according to their trueness,
i.e., $\o < \q < \x$ \emph{(trueness order}, see
\cref{fig:trueness_order_on_values}), and we can order them according
to the amount of information they represent, i.e., $\q < \x$ and
$\q < \o$ (\emph{information order}, see
\cref{fig:information_order_on_values}).  In the latter case the
values $\x$ and $\o$ are incomparable.
\begin{figure}[H]
  \captionsetup[subfigure]{font=footnotesize} \centering
    \subcaptionbox{trueness order\\(Kleene-Logic)\label{fig:trueness_order_on_values}}
    [.49\textwidth]{
        \begin{tikzpicture}[mystyle/.style={draw,circle,inner sep=0pt,minimum size=1.25em}]
        \node [mystyle] (x) at (0,2) {$\x$};
        \node [mystyle] (q) at (0,1) {$\q$};
        \node [mystyle] (o) at (0,0) {$\o$};
        \draw (o) -- (q);
        \draw (q) -- (x);
      \end{tikzpicture}}
    \subcaptionbox{information order\\ (Kripke-Semantics)\label{fig:information_order_on_values}}
    [.49\textwidth]{
        \begin{tikzpicture}[mystyle/.style={draw,circle,inner sep=0pt,minimum size=1.25em}]
        \node [mystyle] (x) at (0.75,1) {$\x$};
        \node [mystyle] (q) at (0,0) {$\q$};
        \node [mystyle] (o) at (-0.75,1) {$\o$};
        \draw (q) -- (o);
        \draw (q) -- (x);
      \end{tikzpicture}}
    \caption{Two orders on the values $\values$}
  \end{figure}
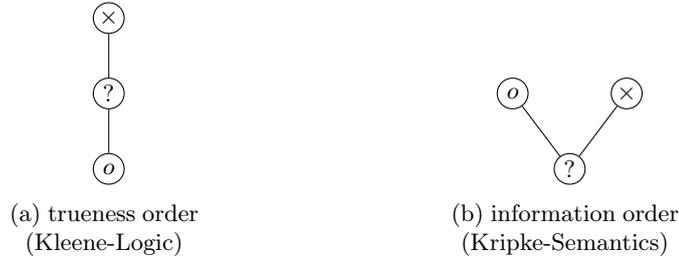
  Both of these orders are useful when thinking about how to evaluate
  formulas in an incomplete context, namely in terms of three-valued
  logics using Kleene-algebras and in terms of Kripke semantics,
  cf. \cref{subsec:attribute_implications}.

  With the Kripke semantics we think of formulas of being
  \emph{certainly valid} (see, \cref{def:valid-formula}) if they hold
  in every \emph{completion} of an incomplete context $\context$,
  i.e., in all contexts where every ``$\q$'' in $\context$ is
  replaced by an ``$\x$'' or  an ``$\o$''.  This is computationally very
  inefficient, since we need to evaluate the formulas in
  $2^{|I^?(\context)|}$ many contexts.  However, we are mainly
  interested in a subset of formulas, the so called \emph{attribute
    implications} (see, \cref{defn:attribute-implication}) where the
  evaluation in the \emph{information order} is equivalent to the
  evaluation in the \emph{trueness order} (see,
  \cref{lem:equivalence-evaluation-attribute-implications-1,lem:equivalence-evaluation-attribute-implications-2}),
  which is far more efficient in terms of computational complexity.

  The \emph{information order} is further useful to define an order on
  incomplete contexts (see, \cref{def:information_order,defn:generalized-information-order}). This
  provides the basis to compare the knowledge of experts (see,
  \cref{defn:comparing-expert-knowledge}) and the results of
  collaborative attribute explorations (see,
  \cref{subsec:comparing-and-evaluating-collab-strategies}).
  Most importantly, it enables combining expert knowledge.

  Now we define an order where we compare two incomplete contexts with regard to the information
  they contain using the information order,
  cf. \cref{fig:information_order_on_values}.
  \begin{defn}[information order, c.f. \cite{holzer2001dissertation}]\label{def:information_order}
    Let $\context_1 \coloneqq \GMWI[1]$ and
    $\context_2 \coloneqq \GMWI[2]$ be two incomplete contexts defined
    on the same object and attribute sets.  We say that $\context_2$
    \emph{contains at least as much information as} $\context_1$,
    abbreviated $\context_1 \leq \context_2$, if
    \begin{align*}
      \forall g \in G,\  \forall m\in M\colon& I_1(g,m)= \x \Implies I_2(g,m)=\x \\
      \tand& I_1(g,m)=\rlap{\:\!$\o$}\phantom{\x}\Implies I_2(g,m)=\o
    \end{align*}
    which is equivalent to
    \[
      \forall g\in G,\ \forall m\in M :\ I_1(g,m) \leq I_2(g,m)
    \]
    where $\leq$ is the information order on $\values$.
  \end{defn}
  
  \begin{rem}
    Note that we are overloading the notation $\leq$. It is used both
    on the value-level and the context-level. 
    We also use the name \emph{information order} for this order on the context-level.
  \end{rem}

  \begin{exmpl} \label{example:information-order}
  Imagine we asked four people
  which of the properties a)~\emph{$\geq$~10~events}, b)~\emph{$\geq$~5~events},
    c)~\emph{female~only~events}, d)~\emph{male~only~events}, and
    e)~\emph{part of $\geq$~8 Olympics} of \emph{Taekwondo} has, and the four incomplete
  contexts $\context_1,\ldots, \context_4$ represent their
  answers.



    \begin{center}
    \begin{cxt}
      \cxtName{$\context_4$} \cxtNichtKreuz{}
      \att{a} \att{b} \att{c} \att{d} \att{e}
      \obj{.xxx?}{Taekwondo}
    \end{cxt}

    \vspace{1ex}
    \begin{cxt}
      \cxtName{$\context_2$} \cxtNichtKreuz{}
      \att{a} \att{b} \att{c} \att{d} \att{e}
      \obj{??xx?}{Taekwondo}
    \end{cxt}
    \begin{cxt}
      \cxtName{$\context_3$} \cxtNichtKreuz{}
      \att{a} \att{b} \att{c} \att{d} \att{e}
      \obj{.xx??}{Taekwondo}
    \end{cxt}

    \vspace{1ex}
    \begin{cxt}
      \cxtName{$\context_1$} \cxtNichtKreuz{}
      \att{a} \att{b} \att{c} \att{d} \att{e}
      \obj{??x??}{Taekwondo}
    \end{cxt}
    \end{center}

  We see that $ \context_4$ \emph{contains at least as much information as}
  $\context_1$, i.e., $\context_1 \leq \context_4$. Further, the
  contexts $\context_2$ and $\context_3$ are incomparable,
  since $\context_2\nleq \context_3$ and  $\context_3\nleq \context_2$.
  \end{exmpl}

  The information order on the set of incomplete contexts for fixed
  sets of objects and attributes, i.e., on the set
  $\values^{G\times M}$, corresponds to 
  the component-wise comparison of all object-attribute incidences
  $(g,m)\in G\times M$ in the information order on $\values$.  In this
  order the infimum for any two such incomplete contexts exists, it is
  the component-wise infimum of the contexts.  Equally, the supremum
  of two incomplete contexts on the same objects and attributes is the
  component-wise supremum.  However, the supremum only exists if the
  contexts do not contain any contradictory information.

  Given two incomplete contexts $\context_{1}$ and $\context_{2}$ the
  \emph{contradictory information} are all pairs $(g,m)$ where an
  object $g$ is known to have an attribute $m$ in one context and
  known not to have it in the other, i.e.,
  $(\relI_1^\x \cap \relI_2^\o) \cup (\relI_1^\o \cap \relI_2^\x)$.
  Note that it is not necessary for the contexts to be defined on the
  same objects and attributes to determine the conflicting
  information.

  The infimum and supremum of two incomplete contexts represent their
  shared and joint information.

  \begin{cor}[c.f. \cite{holzer2001dissertation}] \label{cor:infimum-supremum-context}
    Let $G$ and $M$ be fixed sets of objects and attributes.  Let
    $\values^{G\times M}$ be the set of all incomplete contexts with
    objects $G$ and attributes $M$.  Let
    $\context_1,\context_2\in \values^{G\times M}$ with
    $\context_1 \coloneqq \GMWI[1]$ and
    $\context_2 \coloneqq \GMWI[2]$.
    \begin{enumerate}[label=\alph*)]
    \item \label{cor:infimum_context} Then $\values^{G\times M}$
      together with the information order $\leq$
      (see~\cref{def:information_order}) forms a $\land$-semilattice
      where the \emph{infimum} of two incomplete contexts $\context_1$
      and $\context_2$ is given by

    $$\context_1\land\context_2 \coloneqq (G,M,\values,I_1 \land I_2)$$ 
    with
    \[
      I_1 \land I_2: G\times M \rightarrow \values
    \]
    where
    \[
      (I_1 \land I_2)(g,m) \coloneqq I_1(g,m) \land
      I_2(g,m)=\begin{cases}
        \x  &\text{ if } I_1(g,m)=\x=I_2(g,m)\\
        \o  &\text{ if } I_1(g,m)=\rlap{\:\!$\o$}\phantom{\x}=I_2(g,m)\\
        \q &\text{ otherwise}
      \end{cases}.
    \]

  \item \label{cor:supremum_context} Further, if the two contexts
    $\context_1$ and $\context_2$ have no conflicting information the
    \emph{supremum} $\context_1\vee\context_2$ exists and is given by
    $$\context_1\lor\context_2 \coloneqq (G,M,\values,I_1 \lor I_2)$$ 
    with
    \[
      I_1 \lor I_2: G\times M \rightarrow \values
    \]
    where
    \[
      (I_1 \lor I_2)(g,m) \coloneqq I_1(g,m) \lor
      I_2(g,m)= \begin{cases}
        \x  &\text{ if } I_1(g,m)=\x \text{ or } I_2(g,m)=\x\\
        \o  &\text{ if } I_1(g,m)=\rlap{\:\!$\o$}\phantom{\x} \text{ or } I_2(g,m)=\o\\
        \q &\text{ otherwise}
      \end{cases}.
    \]
  \end{enumerate}
\end{cor}

\begin{exmpl}
  Recall the four incomplete contexts $\context_1 ,\ldots,
  \context_4$ from \cref{example:information-order}.
  We have $\context_1 = \context_2 \land \context_3$.
  Further, the contexts $\context_2$ and $\context_3$ contain no
  conflicting information and their supremum is  $\context_4 = \context_2 \lor \context_3$. 
\end{exmpl}

\begin{rem}
%
  We later show that given an incomplete context $\context_0$ the set
  of $\{\context \mid \context \leq \context_0\}$ is a bounded lattice
  with respect to the information order and that the operators $\land$
  and $\lor$ coincide with the supremum and infimum induced by the
  information order on incomplete contexts, see
  \cref{cor:lattices_are_the_same}.
\end{rem}

\subsection{Attribute Implications}
\label{subsec:attribute_implications}
For the rest of this paper the set of attributes $M$ is considered to
be finite.  One fundamental concept in FCA is that of \emph{attribute
  implications}.  They are used to describe dependencies between
attributes of a formal or incomplete context and are defined as
propositional formulas over an attribute set $M$ in the following way:

\begin{defn}[formulas, models, c.f. \cite{holzer2004knowledgeP1}]\label{def:formulas}
  We define $\formulas{M}$ as the set of propositional formulas over
  $M$ where $M$ is the set of propositional variables.  Let
  $\alpha \in \formulas{M}$ and $B\subseteq M$.  We say $B$ \emph{is a
    model of} $\alpha$ or equally $B$ \emph{respects} $\alpha$ if the
  interpretation of $\alpha$ is true for the valuation
  $v_B : M \to \{\mathrm{true}, \mathrm{false}\}$ with
  $v_B(m) =\mathrm{true}\: :\Leftrightarrow m\in B$.

  For a set of formulas $P \subseteq \formulas{M}$ define the set of
  models of $P$ by
  $\respectingmodels{P} \coloneqq \{ B\subseteq M \mid B \text{
    respects each } \alpha \in P \}$.
    
  For $A\subseteq M$ define $\langle A \rangle_P \coloneqq \bigcap
  \{B\in \respectingmodels{P} \mid A\subseteq B \}$.
\end{defn}
%
The set of \emph{attribute implications} over a set of attributes
$M$ is a specific subset of the propositional formulas $F(M)$:

\begin{defn}[attribute implication, $A\implies B$, c.f. \cite{conf/iccs/BurmeisterH00,holzer2004knowledgeP1}]\label{defn:attribute-implication}
  For $S\subseteq M$ we let $\bigwedge S \coloneqq (s_1\land \ldots
  \land s_n)$ if $S = \{s_1, \ldots ,s_n\}$ and $\bigwedge S \coloneqq
  \mathrm{true}$ if $S=\emptyset$.  For $A\subseteq M$ and $B
  \subseteq M$ we write
  \[
    A\implies B \ \ \text{ for }\ \ \bigwedge A \implies \bigwedge B
  \]
  and call this formula \emph{attribute implication} or short
  \emph{implication}.  If
  $A=\{a_1,\ldots,a_m\}$ we also write $a_1\ldots a_m$ instead of
  $A$ and if $B=\{b_1,\ldots,b_n\}$ we also write $b_1\ldots
  b_n$ instead of $B$, e.g., we write $a_1\ldots a_m \implies
  b_1\ldots b_n$ instead of $A\implies B$.  Further,
  $A$ is referred to as \emph{premise} and
  $B$ as \emph{conclusion} of the implication.
  We abbreviate the set of all implications over the attribute set
  $M$ by $$\Impm \coloneqq \{ A\implies B \mid A,B \subseteq M \}.$$
\end{defn}
%
Note that a set $C\subseteq
M$ respects the attribute implication $A\implies
B$ if and only if $A\subseteq C$ implies $B\subseteq C$.
\begin{defn}\label{def:armstrong-rules-and-cons}{\hspace{-0.5em}\footnote{Note
      that Holzer used a different but equivalent set of rules to
      define the consequence operator in
      \cite{holzer2001dissertation,holzer2004knowledgeP1,holzer2004knowledgeP2},
      however, it is common to utilize the Armstrong rules,
      cf. \cite{GanterWille1999,conf/iccs/BurmeisterH00}.}\ }
  For a set
  $\mathcal{L}$ of attribute implications over an attribute set
  $M$ we define
  $Cons{(\mathcal{L})}$ as the set of all implications obtainable from
  $\mathcal{L}$ by using the Armstrong rules
  \cite{conf/ifip/Armstrong74} (for sets $A,B,C,D$)
  \begin{align*}
    \frac{}{A \implies A}, &&\frac{A \implies C}{A \cup B \implies C}, && \frac{A\implies B \quad B \cup C \implies D}{A\cup C \implies D}.     
  \end{align*}
  \vspace{1ex}
\end{defn}
\begin{defn}[valid formula,
  c.f. \cite{holzer2004knowledgeP1}]\label{def:valid-formula}
  Given a formal context $\context = \GMI$ we call a formula $\alpha
  \in \formulas{M}$ \emph{valid} for an object $g\in G$ if
  $g'$ is a model of $\alpha$.  The formula is valid in
  $\context$ if it is valid for all objects $g\in
  G$.  An attribute implication $A\implies B \in
  \Impm$ is valid in $\context$ if and only if every object $g\in
  G$ that has all the attributes in
  $A$ also has all the attributes in $B$.  We then say
  $B$ \emph{follows} from $A$ in $\context$.
\end{defn}
In the case of incomplete contexts, i.e., three-valued contexts, there
exist many different logics to evaluate formulas, e.g., the
Kleene-Logic~\cite{Kleene1968-KLEITM-3} and other three-valued logics,
cf. \cite{Pagliani1997,holzer2001dissertation,conf/iccs/BurmeisterH00}.
Here we use the \emph{Kripke-semantics}.
\begin{defn}[certainly valid, satisfiable,
  c.f. \cite{conf/iccs/BurmeisterH00,holzer2004knowledgeP1}]
  Given an incomplete context
  $\context=\GMWI$ and a formula $\alpha\in
  \formulas{M}$.  A formal context
  $\tilde\context$ is a \emph{completion} of an incomplete context
  $\context$ if $\context \leq \tilde\context$.  The formula
  $\alpha$ is \emph{Kripke-valid} or \emph{certainly valid} if it is
  valid in every completion of \context.  Further the formula
  $\alpha$ is \emph{satisfiable} or \emph{possibly valid} if it is
  valid in at least one completion of \context.
\end{defn}
\begin{rem}
  For a complete context both \emph{certain} and \emph{possible
    validity} are equivalent and coincide with the usual formulation
  (as in~\cref{def:valid-formula}) of valid formulas for formal
  contexts.
\end{rem}
\begin{exmpl}
  As an example recall the incomplete context in
  \cref{fig:incomplete_context_example}.  Here the implication $b
  \implies
  d$ is certainly valid as it is valid in every completion of the
  context, whereas $c \implies
  e$ is satisfiable but not certainly valid since `Taekwondo' has
  `female only events' but could or could not be `part of at least
  eight Olympic Games'.
\end{exmpl}

\begin{defn}[$\Imp$, $\Sat$,
  c.f. \cite{holzer2004knowledgeP1,holzer2001dissertation}]\label{def:Imp-Sat}
  Given an incomplete context $\context$ we denote the set of all
  certainly valid implications by
    $$\Imp[\context] \coloneqq  \{A\implies B \in \Impm \mid A\implies B \text{ is certainly valid in } \context \}$$
    and the set of all satisfiable implications by
    $$\Sat[\context] \coloneqq \{ A\implies B \in \Impm \mid A \implies B \text{ is satisfiable in } \context \}.$$
  \end{defn}

  For the rest of this section we recollect some facts about attribute
  implications in the case of incomplete contexts.  The set of
  certainly valid implications is closed with respect to the Armstrong
  rules.
  \begin{thm}[see
    \cite{holzer2004knowledgeP1,holzer2001dissertation,GanterWille1999}]
    With $Cons(\cdot)$ and $Imp(\cdot)$ as defined in
    \cref{def:armstrong-rules-and-cons,def:Imp-Sat} we have
    $\Cons{\Imp[\context]} = \Imp[\context]$ for every incomplete
    context $\context$.
  \end{thm}

  The operators $\cdot^\Box$ and $\cdot^\Diamond$ can be used to
  efficiently compute whether an attribute implication is certainly
  valid or satisfiable.  This corresponds to the evaluation in
  Kleene-Logic
  \cite{conf/iccs/BurmeisterH00,holzer2001dissertation,Pagliani1997}.

  The following lemmas clarify that implicational formulas can be
  evaluated in Kripke semantics and Kleene-Logic giving the same
  result.  Note that for arbitrary attribute formulas this is not
  true.
  \begin{lem} [see \cite{holzer2004knowledgeP1} Lemma
    5] \label{lem:equivalence-evaluation-attribute-implications-1}
    Let $\context = \GMWI$ be an incomplete context and
    $A,B \subseteq M$. Then the following conditions are equivalent:
    \begin{enumerate}
    \item $A\implies B \in \Imp[\context]$
    \item $B\setminus A \subseteq A^{\Diamond\Box}$
    \item $A^\Diamond \subseteq (B\setminus A)^\Box$
    \item For all $g\in G$ with $A\subseteq g^\Diamond$ we have
      $B\setminus A \subseteq g^\Box$
    \item For all $g\in G$ holds: If $I(g,a)\neq \o$ for all $a\in A$
      then $I(g,b)=\x$ for all $b\in B\setminus A$.
    \end{enumerate}
  \end{lem}

  \begin{lem}[see \cite{holzer2004knowledgeP1} Lemma
    6] \label{lem:equivalence-evaluation-attribute-implications-2}
    Let $\context = \GMWI$ be an incomplete context and
    $A,B \subseteq M$. Then the following conditions are equivalent:
    \begin{enumerate}
    \item $A\implies B \in \Sat[\context]$
    \item $B \subseteq A^{\Box\Diamond}$
    \item $A^\Box \subseteq B^\Diamond$
    \item For all $g\in G$ with $A\subseteq g^\Box$ we have
      $B\subseteq g^\Diamond$
    \item For all $g\in G$ holds: If $I(g,a)=\x$ for all $a\in A$ then
      $I(g,b)\neq \o$ for all $b\in B$.
    \end{enumerate}
  \end{lem}

  Further the operator $\cdot^{\Box\Diamond}$ is useful to compute the
  maximal satisfiable conclusion for a premise.
  \begin{lem}[see \cite{holzer2004knowledgeP1} Corollary 6]
    Let $\context = \GMWI$ be an incomplete context and
    $A\subseteq M$. Then
    $$ A^{\Box\Diamond} \coloneqq \{m\in M \mid A\implies m \in \Sat[\context]\}.$$ 
  \end{lem}

  \subsection{Attribute Exploration}
  \label{subsec:attribute_exploration}
%
%
%
  Let $\context^{U}=(G^U,M,I^U)$, $|M|<\infty$, be an (unknown) formal
  context called universe.  This context represents the knowledge
  domain of interest.  (We assume that the domain can be represented
  in such a way).  The so called \emph{attribute exploration} is an
  interactive algorithm that helps an expert gain maximum insight into
  the dependency structure of the domains attributes.

  The following algorithm is taken from \cite{holzer2001dissertation}
  and condensed to the main steps.  It describes the process of
  exploring a knowledge domain modelled as an unknown formal context
  $\context^U$ using the knowledge of an expert in an algorithmic
  fashion.  Here we assume that the expert's answers are always
  consistent with the domain, i.e., an accepted implication is valid
  in $\context^U$ and given counterexamples are objects of the domain
  contradicting the implication in question.  The attribute
  exploration produces a set of valid implications in the universe
  $\context^U$ and a list of counterexamples against non-valid
  implications.  The following algorithm describes the exploration:

  \begin{itemize}
  \item[(E1)] At the beginning of the exploration algorithm the user
    enters the (finite) set of attributes $M$ whose dependencies are
    to be explored.
    
  \item[(E2)] Let $j\coloneqq 1$.  The set of accepted implications is
    initialized with the empty set $P_1 \coloneqq \emptyset$.  The
    context of examples is initialized with an empty incomplete
    context $\context_1 \coloneqq (G_1=\emptyset,M,\values,I_1)$ with
    $I_1: \emptyset \times M \rightarrow \values$.

  \item[(E3)] The set $P_j$ contains the implications accepted as
    valid so far.  In the $j$-th step the algorithm chooses an
    implication $A\implies B$ that might be valid in $\context^U$,
    such that the set $A\subseteq M$ is minimal (w.r.t. $\subseteq$),
    respects $P_j$ and
    $B:=A^{\Box\Diamond}=\{m\in M \mid A\implies m \in
    \Sat[\context_j]\}\neq A$.  If the implication is derivable from
    $P_j$ it is accepted automatically.  Otherwise the program asks
    the expert whether $A\implies B$ is valid in the universe
    $\context^U$.  The expert can answer \emph{yes}, \emph{no} or
    \emph{unknown}:
  
    \begin{itemize}
    \item[(yes)] The implication $A\implies B$ is accepted as valid
      and added to the set of accepted implications:
      $P_{j+1} \coloneqq P_j \cup \{A\implies B\}$.  Let
      $\context_{j+1} = \context_j$.
        
    \item[(no)] The expert must give at least one counterexample
      $g\in G^U$ against the implication $A\implies B$.  For each
      counterexample she enters the context row of $g$ which may
      contain question marks, i.e., unknown relations between $g$ and
      some attributes.  Let $P_{j+1}\coloneqq P_j$ and
      $\context_{j+1}$ be the context $\context_j$ after adding the
      rows of all counterexamples $g$.
        
    \item[(unknown)] The expert is asked for which attributes $b\in B$
      the implication $A\implies b$ is unknown.  Let
      $Z\coloneqq \{b\in B \mid A\implies b \text{ is unknown}\}$.
      For $b\in B\setminus Z$ the implication $A\implies b$ is valid
      in the universe $\context^U$, because every counterexample
      against $A\rightarrow b$ would be a counterexample against
      $A\implies B$.
        
      For $b\in Z$ the algorithm checks whether
      $A\implies b \in \Cons{P_j \cup \{A\implies B\setminus Z\}}$
      holds.  In case it holds $b$ can be removed from $Z$, since
      $A\implies b$ follows from implications known to be valid in the
      universe $\context^U$ and must therefore also be valid.  In case
      it does not hold for $b$, i.e.,
      $A\implies b \notin \Cons{P_j \cup \{A\implies B\setminus Z\}}$,
      fictitious objects are added to $\context_{j}$.
        
      For each remaining $b\in Z$ we add the ficticious object
      $g^?_{A\not \implies b}$ that contradicts the implication
      $A\implies b$, i.e., $g^?_{A\not \implies b}$ has all attributes
      in $A$, does not have the attribute $b$ and the relation to all
      other attributes is unknown.  We assume that
      $g^?_{A\not \implies b}$ is a new object, i.e.,
      $ g^?_{A\not \implies b} \notin G^U$ and
      $ g^?_{A\not \implies b} \notin G_j$.
        
      Let
      $\context_{j+1} \coloneqq (G_j \cup \{ g^?_{A\not \implies b}\mid
      b\in Z\},M,\values,J)$ with $J(g,m) = I_j(g,m)$ for all
      $g \in G_j$, $m\in M$ and for $b\in Z$ let
      $J( g^?_{A\not \implies b},a) = \x$ for $a\in A$,
      $J( g^?_{A\not \implies b}, b)=\o$ and
      $J( g^?_{A\not \implies b},m)=\q $ for
      $m\in M\setminus (A\cup \{b\})$.

      Let $P_{j+1}\coloneqq P_j \cup \{A \implies B\setminus Z\}$ if
      $B\setminus Z \neq A$ and $P_{j+1} \coloneqq P_j$ if
      $B\setminus Z = A$.
    \end{itemize}
    
  \item[(E4)]
%
    If every set $A$ that respects $P_j$ and is not already a premise
    in $P_j$ satisfies
    $A=\{m\in M \mid A\implies m \in \Sat[\context_j]\}$ the algorithm
    ends.  Otherwise increment $j$, i.e., let $j\coloneqq j+1$, and
    repeat the steps (E3) and (E4).
    
  \end{itemize}


  This is a stripped-down version of attribute exploration for
  incomplete knowledge (without handling of background knowledge and
  reductions of question marks), because most of the exploration
  procedure itself is of no particular interest for the remainder of
  this paper.
  An example of how this algorithm works in detail is beyond the scope
  of this paper and we refer the reader to
  \cite{conf/iccs/BurmeisterH00} and
  \cite{holzer2001dissertation}.

  There already exists theory on more advanced techniques such as the
  use of background knowledge and reductions of question marks based
  on already accepted implications.  For exploration with incomplete
  knowledge more information can, for example, be found in
  \cite{conf/iccs/BurmeisterH00,holzer2001dissertation,holzer2004knowledgeP1}.
  Other modifications such as allowing exceptions to attribute
  implications \cite{stumme96attribute} and background knowledge in
  the form of implications and clauses
  \cite{stumme96attribute,ganter1999attribute,ganter2016conceptual}
  seem adaptable to attribute exploration with incomplete knowledge as
  well.

  \begin{fact}[cf. \cite{holzer2004knowledgeP1,holzer2004knowledgeP2,conf/iccs/BurmeisterH00}]
    \label{fact:attribute_exploration_results_in_maximum_knowledge_w_r_t_expert_knowledge}
    At the end of the attribute exploration the result contains
    maximal information (with respect to the expert's knowledge) about
    implications of the unknown universe $\context^U$.  Assuming the
    exploration ended after $j$ steps, the result consists of
    \begin{enumerate}[topsep=0pt]
    \item a list of implications $P_j$ that are known to be valid,
    \item a list of fictitious counterexamples
      $G^* \coloneqq G_j \setminus G^U$ contradicting implications
      where the expert answered `unknown',
    \item a list of counterexamples $G_j\setminus G^*$ contradicting
      the implications which are known not be valid and
    \item a list of implications
      $P_j \cup \{A\implies b \mid g^?_{A\not \implies b} \in G^*\}$
      which are possibly valid.
    \end{enumerate}
    To obtain the complete knowledge about the domain it now suffices
    to check all implications that were answered by `unknown' before.
    If for each of these implications it can be decided whether they
    are valid or have to be rejected, complete knowledge about the
    domain is received:
    
    An implication is valid in $\context^U$ if and only if it is
    derivable from the implications accepted as valid and implications
    rejected as `unknown' that in fact are valid in $\context^U$.
  \end{fact}
  So far we have presented known results.
  We have seen that attribute exploration works with a single (reliable) expert
  who can respond with partial knowledge to questions posed by the
  exploration algorithm.
  \cref{fig:basic_exploration} visualizes this
  part of the exploration process.

\begin{figure}[H]
  \centering
  \begin{tikzpicture}
    \tikzset{
      >=stealth',
      box/.style={ rectangle, rounded corners, draw=black, very thick,
        text width=6.5em, minimum height=1cm, text centered, node
        distance=5.5cm}, label/.style={ text width=8.5em, text
        centered, },
      exp/.style={ ellipse, draw=black, very thick,
        minimum height=2em, text centered, node distance=4cm},
      pil/.style={ ->, thick, shorten <=4pt, shorten >=4pt,}} \node
    [box, draw](exploration){Attribute Exploration}; \node [exp, draw,
    right of=exploration](expert1){Expert};
    
    \path (exploration) edge [pil,bend left] node[label, yshift=0.5cm]
    {Question:\\ $A\implies B$ valid ?} (expert1); \path (expert1)
     edge [pil,bend left] node[label, yshift=-0.5cm] {Answer:\\ yes,
       no, unknown} (exploration);
    
   \end{tikzpicture}
   \caption{Visualization of exploration with one expert}
   \label{fig:basic_exploration}
 \end{figure}
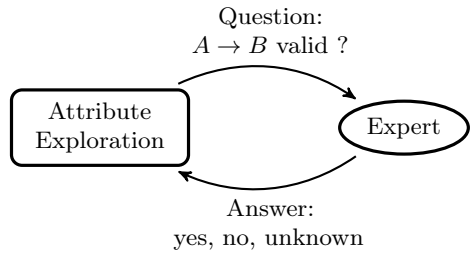

 Let us recap: The algorithm of the attribute exploration under
 incomplete knowledge generates questions that are to be answered by
 the domain expert.  The expert is not omniscient but reliable, i.e.,
 the answers she gives are consistent with the true knowledge in the
 domain.
 The result is dependent on the knowledge of the expert conducting the
 exploration.
  The answers `no' and `unknown' come with additional
 information provided by the expert, i.e., with counterexamples or the
 set of attributes for which is unknown if they follow from the
 premise.  The result of such an attribute exploration is a set of
 attribute implications known to be valid, an incomplete context of
 counterexamples that contradicts implications that are known to be
 invalid in the domain and a set of fictitious counterexamples that
 contradict implications that were unknown to the expert.

 \section{Experts and Collaboration}
  \label{sec:experts_and_collaboration}

 Based on previous works on attribute exploration
 under incomplete knowledge (c.f. \cref{subsec:attribute_exploration})  we now modify the attribute exploration to work with multiple
 experts.  The idea is that instead of a single expert who answers
 the questions directly we have a strategy to answer the questions
 with help of multiple experts -- see \cref{fig:collab_exploration}.
 Once again the answers `no' and `unknown' come with additional
 information.  Namely, some counterexamples and the set of attributes
 $Z\subsetneq B$ for which the implication $A\implies Z$ is unknown and
 $A\implies B\setminus Z$ is valid.  Note that we slightly modify the
 attribute exploration to receive an answer that already contains all
 the additional information expected by the algorithm.

\begin{figure}[H]
  \centering
  \begin{tikzpicture}
    \tikzset{
      >=stealth',
      box/.style={ rectangle, rounded corners, draw=black, very thick,
        text width=6.5em, minimum height=2cm, text centered, node
        distance=5.5cm}, label/.style={ text width=8.5em, text
        centered, },
      exp/.style={ ellipse, draw=black, very thick,
        minimum height=2em, text centered, node distance=4cm},
      pil/.style={ ->, thick, shorten <=4pt, shorten >=4pt,}} \node
    [box, draw](exploration){Attribute Exploration}; \node [box, draw,
    right of=exploration, minimum height= 2cm](strategy){Collaboration
      Strategy}; \node [exp, draw, right
    of=strategy,yshift=1.5cm](expert1){Expert}; \node [exp, draw,
    right of=strategy,yshift=0.5cm](expert2){Expert}; \node [exp,
    draw, right of=strategy,yshift=-0.5cm](expert3){Expert}; \node
    [below of=expert3,yshift=0.25cm](further-experts){\vdots};
    
    \path (exploration) edge [pil,bend left] node[label, yshift=0.5cm]
    {Question:\\ $A\implies B$ valid ?} (strategy); \path (strategy)
     edge [pil,bend left] node[label, yshift=-0.5cm] {Answer:\\ yes,
       no, unknown} (exploration);
    
     \path (strategy.north east) edge [pil,bend left=5] (expert1);
     \path (expert1) edge [pil,bend left=5]
     ($(strategy.north east)-(0,0.2)$);
    
     \path ($(strategy.east)+(0,0.3)$) edge [pil,bend left=5]
     (expert2); \path (expert2) edge [pil,bend left=5]
     ($(strategy.east)+(0,0.1)$);
    
     \path ($(strategy.east)-(0,0.3)$) edge [pil,bend left=5]
     (expert3); \path (expert3) edge [pil,bend left=5]
     ($(strategy.east)-(0,0.5)$);
    
     \draw[thick,dotted] ($(strategy.north west)+(-0.00,1)$) rectangle
     ($(further-experts.south east)+(0.85,0)$)
     node[midway,above=1.8cm] {Expert Collaboration};
    
   \end{tikzpicture}
   \caption{Visualization of attribute exploration with multiple
     experts}
   \label{fig:collab_exploration}
 \end{figure}
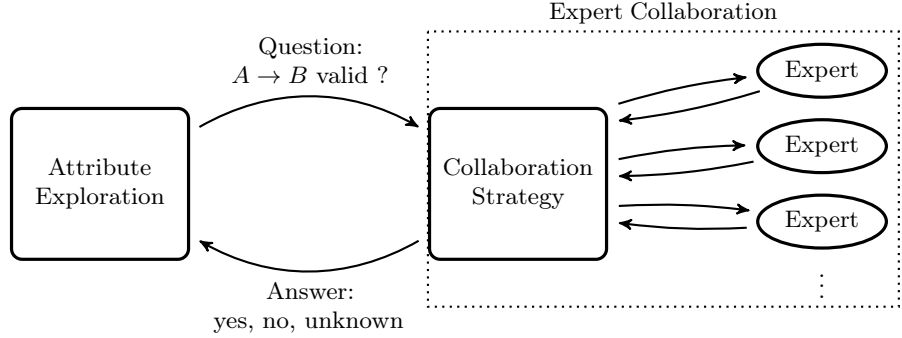

 With this picture in mind we begin by formalizing the expert
 .  We
 generalize the information order to work on different object sets to
 be able to compare and combine the knowledge of experts.
%
 We then adapt the idea of a consortium \cite{conf/iccs/HanikaZ18} to
 a group of experts, formulate a notion of collaboration, give a few
 examples and proceed to discuss methods to compare different
 collaboration strategies.

 An example of three experts with incomplete knowledge conducting a collaborative exploration of properties of the
 \emph{Disciplines of the Summer Olympic Games 2020} can be found in the
 appendix (see \cref{example:appendix}).

 \subsection{Expert Knowledge}

 The expert, a key component of every attribute exploration, is often
 only described as an entity that correctly answers the posed
 questions, especially in many of the earlier works on the subject.
 Later works, e.g, \cite{ganter2016conceptual,conf/iccs/HanikaZ18},
 make an effort to also model the expert in a formal way, usually as a
 mapping from the set of attribute implications into the target domain
 represented as closure system over $M$.  In the following we suggest
 a model for experts in an incomplete knowledge setting where we
 encode the knowledge of an expert by both an incomplete context of
 examples and a set of valid implications.
 We then model a notion of interaction with the expert.

 First, we formally introduce the \emph{universe},
 cf. \cref{subsec:attribute_exploration}, which represents complete
 (but unknown) knowledge of the domain of interest, i.e., the domain
 about which the expert has knowledge.  We assume that the universe
 can be represented as a formal context.
 \begin{defn}[universe]
   In the following let $\context^U=\GMI=(G,M,\{\x, \o\},I)$,
   $|M|<\infty$ always be a formal context which we call
   \emph{universe}.  The set $\mathcal{L}\coloneqq \Imp[\context^U]$
   denotes the set of valid implications in the \emph{universe}.
 \end{defn}
 Note that the the universe could equally be defined as a formal or
 incomplete context.  We chose to restrict ourselves to a universe
 represented by a formal context for the sake of readability.

 \begin{defn}[expert knowledge]
   \emph{Expert knowledge} about the universe $\context^U$
    $$E=(\context_E,\Cons{\mathcal{L}_E})$$
    consists of a context $\context_E \coloneqq (G_E,M,\values,I_E)$,
    $G_E \subseteq G$, of objects that are (partially) known to the
    expert, i.e., $\context_E \leq \context^U|_{G_E}$, and a set of
    implications
    $\mathcal{L}_E \subseteq \mathcal{L} \coloneqq \Imp[\context^U]$
    that the expert knows to be valid.
    
  
    Note that we also use the terms \emph{expert for a domain} and
    \emph{expert for a universe} to indicate that an expert has expert
    knowledge about a universe.
  \end{defn}
  By definition the set of partial counterexamples and the set of
  known valid implications are compatible:
  \begin{lem}\label{lem:expert_knowledge_is_compatible}
    Let $E=\expert{E}$ be expert knowledge about a universe
    $\context^U$.  Then every implication in $\Cons{\mathcal{L}_E}$ is
    satisfiable in $\context_{E}$.
  \end{lem}
  \begin{proof}
    Let $A\implies B \in\Cons{\mathcal{L}_E}$.  Since
    $\mathcal{L}_E \subseteq \mathcal{L}$ and
    $\Cons{\mathcal{L}}=\mathcal{L}$ we know that $A\implies B$ is
    valid in $\context^{U}$.  Therefore $\context^U$ does not contain
    any counterexamples for $A\implies B$ and $A\implies B$ is valid
    in every subcontext $(T,M,I|_{T\times M})$ with $T \subseteq G$ of
    $\context^U$.  Since $\context_{E} \leq \context^U|_{G_E}$ and
    $A\implies B$ is valid in $\context^U|_{G_E}$ we have that
    $A\implies B$ is satisfiable in $\context_{E}$.
  \end{proof}

  One might be tempted to use the example knowledge to infer further
  implications, but much like in reality this is not justified.  The
  certainly valid implications of the expert's example context are not
  necessarily valid in the universe.  Furthermore, the valid
  implications known to the expert are not necessarily certainly valid
  with regard to the set of partial counterexamples known to the
  expert as shown in
  \cref{remark:validity_implications_examples_expert}.
  \begin{rem}
    \label{remark:validity_implications_examples_expert}
    Assume that we have an expert $\expert{E}$ for the universe
    $\context^U$ then
    the implications in $\mathcal{L}_E$ are not necessarily certainly
    valid in $\context_{E}$.  Consider for instance the following
    example (cf. \cref{fig:formal_context_examples}):
    \begin{center}
      \begin{cxt}
        \cxtName{$\context^U$} \cxtNichtKreuz{}
        \att{a} \att{b} \att{c}
        \obj{xxx}{Aquatics -- Swimming} \obj{.xx}{Badminton}
        \obj{..x}{Gymnastics -- Rhythmic}
      \end{cxt}
      \begin{cxt}
        \cxtName{$\context_E$} \cxtNichtKreuz{} \att{a} \att{b}
        \att{c} \obj{?.?}{Gymnastics -- Rhythmic}
      \end{cxt}
%
%
%
%
    \end{center}
    Let
    $\{a\implies c\} = \mathcal{L}_E \subseteq \mathcal{L}=
    \Imp[\context^U]$, then $a\implies c$ is satisfiable but not
    certainly valid in $\context_{E}$.  Further, the implication
    $b\implies a$ is certainly valid in $\context_E$ but not valid in
    $\context^U$.
  \end{rem}

\begin{rem}
  In our setting (and in contrast to \cite{conf/iccs/HanikaZ18}),
  multiple experts for a universe refer to the same objects with the
  same names.  By definition there is no confusion about objects.
  Different experts for a universe can know about different objects or
  different attribute object relations.  However, they can not have
  conflicting knowledge,
  cf. \cref{lem:expert_knowledge_of_a_group_of_experts_is_compatible}.
\end{rem}

\subsection{Generalized Information Order}

We want to compare and combine the knowledge of different experts.  To
achieve this we need to be able to compare both the known examples and
the known implications.  The known implications can easily be compared
using the set inclusion order $\subseteq$.  However, to compare the
known examples of different experts we need to generalize the
information order to allow comparing contexts with different object
sets.

\begin{defn}[generalized information order]\label{defn:generalized-information-order}
  Given two incomplete contexts $\context_{1}=(G_1,M,\values,I_1)$ and
  $\context_{2}=(G_2,M,\values, I_2)$ on object sets
  $G_1,G_2 \subseteq G$, we say that $\context_{2}$ \emph{contains at
    least as much information as} $\context_{1}$, abbreviated
  $\context_{1} \leq_g \context_{2}$, if
    $$G_1 \subseteq G_2 \tand \context_{1} \leq \context_{2}|_{G_1}.$$
    where $\leq$ is the information order on incomplete contexts, see
    \cref{def:information_order}.
  \end{defn}
  Obviously we have $\leq_g = \leq$ if we compare incomplete contexts
  that have the same object and attribute sets, therefore we use
  \emph{contains at least as much information} in both cases.

  The infimum of two incomplete contexts on the same attribute set
  with respect to the generalized information order always exists:
  \begin{defn}[generalized information infimum]
    Given two incomplete contexts $\context_{1}=\GMWII[1]$ and
    $\context_{2}=\GMWII[2]$ on object sets $G_1,G_2 \subseteq G$, the
    \emph{generalized information infimum}
    $\context_{1} \land_g \context_{2} $ is defined by
    \[
      \context_{1} \land_g \context_{2} \coloneqq (G_1 \cap G_2,
      M,\values, I_1 \land I_2)
    \]
    where
    \[
      I_1 \land I_2:(G_1 \cap G_2)\times M \rightarrow \values
    \]
    with $(I_1 \land I_2)(g,m) = I_1(g,m)\land I_2(g,m)$ defined as
    before, see \cref{cor:infimum-supremum-context}
    \labelcref{cor:infimum_context}.
  \end{defn}

  Again, for incomplete contexts on the same attribute set with no
  conflicting information there exists a supremum with respect to the
  generalized information order:
  \begin{defn}[generalized information supremum]
    Given two incomplete contexts $\context_{1}=(G_1,M,\values,I_1)$
    and $\context_{2}=(G_2,M,\values, I_2)$ on object sets
    $G_1,G_2 \subseteq G$, with no conflicting information the
    \emph{generalized information supremum}
    $\context_{1} \lor_g \context_{2}$ is defined by
    \[
      \context_{1} \lor_g \context_{2} \coloneqq (G_1 \cup G_2,
      M,\values, I_1 \lor I_2)
    \]
    where
    \[
      I_1 \lor I_2:(G_1 \cup G_2)\times M \rightarrow \values
    \]
    with $(I_1 \lor I_2)(g,m) = I_1(g,m)\lor I_2(g,m)$ defined as
    before, see \cref{cor:infimum-supremum-context}
    \labelcref{cor:supremum_context}, with the addition that we extend
    the domains of $I_1$ and $I_2$ to $G_1\cup G_2$, each by mapping
    previously undefined object-attribute combinations to
    $\q$.
  \end{defn}

  The following lemma and corollaries allow us to compare and combine
  example knowledge.
  \begin{lem}
    \label{lem:lattices_are_the_same}
    The set $S$ of all incomplete contexts that are derived from an
    incomplete context $\context^U=\GMWI$,
    $S := \{\context \mid \context \leq_g \context^U\}$, ordered by
    the generalized information order constitutes a bounded lattice
    where $\land_g$ and $\lor_g$ define the infimum and supremum.
  \end{lem}
  \begin{proof}
    The infimum of any two contexts from $S$ exists and is the infimum
    on the restrictions of the contexts to the set of shared objects,
    cf. \cref{cor:infimum-supremum-context}
    \labelcref{cor:infimum_context}.  There is no conflicting
    information for any two contexts in $S$.  This directly follows
    from the definition of $S$ where every context contains partial
    information of the same context $\context^U$.  Therefore the
    supremum on the set of shared objects always exists,
    cf. \cref{cor:infimum-supremum-context}
    \labelcref{cor:supremum_context}, which can be extended to the set
    of combined objects by using the corresponding valuations of any
    object-attribute pair where the object only appears in exactly one
    of the contexts.
    Further the incomplete context that contains no objects
    $\{\emptyset,M,\values,I_\emptyset\}$ is in $S$ and is the lower
    bound for all infima of incomplete contexts in $S$.  Equally the
    context $\context^U$ is the upper bound for all suprema of
    incomplete contexts in $S$.
    
    Now we show that for $\context_{1},\context_{2}\in S$ where
    $\context_{1}=\GMWII[1]$ and $\context_{2}=\GMWII[2]$ that
    \begin{enumerate}[label=\arabic*)]
    \item
      $\context_{1} \leq_g \context_{2} \Leftrightarrow \context_{1}
      \land_g \context_{2} = \context_{1}$ and
    \item
      $\context_{1} \leq_g \context_{2} \Leftrightarrow \context_{1}
      \lor_g \context_{2} = \context_{2}$.
    \end{enumerate}
    which shows that the infimum and supremum in $S$ with respect to
    $\leq_g$ coincide with $\land_g$ and $\lor_g$.
    
    1) `$\Rightarrow$': Let $\context_{1} \leq_g \context_{2}$.  Then
    $G_1 \subseteq G_2$ and $\context_{1} \leq \context_{2}|_{G_1}$.
    Hence $G_1 \cap G_2 = G_1$ and
    $\forall (g,m)\in G_1 \times M : I_1(g,m) = I_2(g,m)$.  Therefore
    $\context_{1} \land_g \context_{2} = \context_{1}$.
    
    1) `$\Leftarrow$': Let
    $\context_{1} \land_g \context_{2} = \context_{1}$.  Then
    $G_1 \cap G_2 = G_1$ and $I_{1} \land I_{2}= I_{1}$.  Hence
    $G_1 \subseteq G_2$ and
    $\forall (g,m)\in G_1 \times M : I_1(g,m) = I_2(g,m)$ particularly
    $I_1(g,m) \leq I_2(g,m)$.  Therefore
    $\context_{1} \leq \context_{2}|_{G_1}$.
    
    2) `$\Rightarrow$': Let $\context_{1} \leq_g \context_{2}$.  Then
    $G_1 \leq G_2$ and therefore $G_1 \cup G_2 = G_2$.  Further
    $\context_{1} \leq \context_{2}|_{G_1}$ implies
    $\forall (g,m)\in G_1\times M: I_1(g,m) \leq I_2(g,m)$.  Therefore
    $I_{1}\lor I_{2} = I_2$ and
    $\context_{1} \lor_g \context_{2} = \context_{2}$.

    2) `$\Leftarrow$': Let
    $\context_{1} \lor_g \context_{2} = \context_{2}$.  Then
    $G_1 \cup G_2 = G_2$ hence $G_1 \subseteq G_2$.  Further
    $I_{1}\lor I_{2} = I_2$ hence
    $\forall (g,m)\in G_1 \times M : I_1(g,m) = \x \Implies
    I_2(g,m)=\x \tand I_1(g,m) = \o \Implies I_2(g,m)=\o$.  Therefore
    $\context_{1} \leq \context_{2}|_{G_1}$ and thus
    $\context_1 \leq_g \context_{2}$.
  \end{proof}
  \begin{cor}
    \label{cor:expert_example_knowledge_forms_a_lattice_when_adding_the_domain_context_and_empty_context}
    Given an incomplete context $\context^U=\GMWI$.  Every subset of
    $\{\context \mid \context \leq_g \context^U\}$ that contains
    $\context^U$ and $\{\emptyset,M,\values,I_\emptyset\}$ is a
    bounded lattice with respect to the generalized information order.
  \end{cor}
  \begin{cor}
    \label{cor:lattices_are_the_same}
    The set $S$ of all incomplete contexts that are derived from an
    incomplete context $\context^U$,
    $S := \{\context \mid \context \leq \context^U\}$, ordered by the
    information order constitutes a bounded lattice where $\land$ and
    $\lor$ define the infimum and supremum.
  \end{cor}
  \begin{proof}
    This follows from $\leq = \leq_g$, $\land = \land_g$ and
    $\lor = \lor_g$ for incomplete contexts on the same object sets
    together with \cref{lem:lattices_are_the_same}.
  \end{proof}

  The following fact and corollary allow us to compare and combine
  implications and implicational knowledge by making use of lattice
  structures and the corresponding infimum and supremum operators.
  \begin{fact}
    Given a formal context $\context^U$.  The power set on the set of
    valid implications $\Imp[\context^U]$ with the subset inclusion as
    order relation forms a lattice where intersection and union define
    infimum and supremum.
  \end{fact}
  \begin{cor}
    Given a formal context $\context^U$ let
    $\mathfrak{X}:=\{\Cons{X} \mid X \subseteq \Imp[\context^U]\}$.
    Then $(\mathfrak{X}, \subseteq)$ is a lattice with infimum
    $\bigcap \mathcal{X}$ and supremum $\Cons{\bigcup\mathcal{X}}$ for
    all $\mathcal{X} \subseteq \mathfrak{X}$.
  \end{cor}
  \begin{proof}
    By definition $\Cons{\cdot}$ is a closure operator on
    $\Imp[\context^U]$, cf. \cite[Prop. 21]{GanterWille1999}, and
    therefore
    $\mathfrak{X}:=\{\Cons{X} \mid X \subseteq \Imp[\context^U]\}$ is
    a closure system.  Using Prop. 3 from \cite{GanterWille1999} we
    obtain this fact.
  \end{proof}
  Note that $\mathfrak{X}$ is the set of possible implicational
  knowledge about a universe $\context^U$.
  \subsection{Comparing and Combining Expert Knowledge}

  The general information order allows us to compare experts in terms
  of their example knowledge.  Together with the set inclusion order
  on implications known by the experts we can now compare expert
  knowledge.

\begin{defn}\label{defn:comparing-expert-knowledge}
  Given a formal context $\context^U = \GMI$,
  $\mathcal{L}=\Imp[\context^U]$ and two experts $E_1$ and $E_2$ on
  $\context^U$ where $E_1\coloneqq \expert{1}$,
  $\context_1 \leq_g \context^U$,
  $\mathcal{L}_1 \subseteq \mathcal{L}$ and $E_2\coloneqq \expert{2}$,
  $\context_2 \leq_g \context^U$,
  $\mathcal{L}_2 \subseteq \mathcal{L}$ we say:
  \begin{enumerate}[label=\alph*)]
  \item $E_2$ has at least as much example knowledge as $E_1$ if
    $\context_1 \leq_g \context_2$,
  \item $E_2$ has at least as much implication knowledge as $E_1$ if
    $\Cons{\mathcal{L}_1} \subseteq \Cons{\mathcal{L}_2}$,
  \item $E_2$ knows at least as much as $E_1$ if $E_2$ has at least as
    much example and implication knowledge as $E_1$.  We denote this
    by $E_1 \leq E_2$.
  \end{enumerate}
\end{defn}

Further, we can combine the knowledge of experts using the
component-wise infimum and supremum on the product lattice of
incomplete example contexts and the implication knowledge lattice:
\begin{defn}
  Given expert knowledge of two experts $E_1$ and $E_2$ of a domain
  $\context^U$.  The maximum combined knowledge, i.e., the supremum of
  example and implication knowledge, is defined by
  \[
    E_1\vee E_2 \coloneqq (\context_1 \lor_g \context_2,
    \Cons{\mathcal{L}_1 \cup \mathcal{L}_2}).
  \]
  The shared knowledge, i.e., the infimum of example and implication
  knowledge, is defined by
  \[
    E_1\wedge E_2 \coloneqq (\context_1 \land_g \context_2, \Cons{
      \mathcal{L}_1 } \cap \Cons{\mathcal{L}_2 } ).
  \]
    
%
  Note that
  $\Cons{\mathcal{L}_1 \cup \mathcal{L}_2} =
  \Cons{\Cons{\mathcal{L}_1} \cup \Cons{\mathcal{L}_2}}$ by definition
  of $\Cons{\cdot}$.
\end{defn}

%

\begin{defn}
  The \emph{maximum combined knowledge} of a group of experts
  $\{E_1,\ldots,E_n\}$ is defined by
  $\bigvee_{i\in\{1,\ldots,n\}} E_i$.
\end{defn}


\begin{rem}
  The maximum combined knowledge serves as a reference point for
  strategies of experts working collaboratively.  However there are
  some limitations.  Consider two experts $E_1=\expert{1}$ and
  $E_2=\expert{2}$ for a formal context $\context^U$
  (cf. \cref{fig:formal_context_examples}) with
  $\Imp[\context^U]=\mathcal{L}$ given by:
  \begin{center}
    \begin{cxt}
      \cxtName{$\context^U$} \cxtNichtKreuz{} \att{a} \att{b}
      \obj{xx}{Aquatics -- Swimming} \obj{.x}{Badminton}
      \obj{..}{Gymnastics -- Rhythmic}
    \end{cxt}
    \begin{cxt}
      \cxtName{$\context_1$} \cxtNichtKreuz{} \att{a} \att{b}
      \obj{xx}{Aq. -- Swim.}  \obj{?.}{Gym. -- Rhy.}
    \end{cxt}
    \begin{cxt}
      \cxtName{$\context_2$} \cxtNichtKreuz{} \att{a} \att{b}
      \obj{.x}{Badminton}
      \obj{.?}{Gym. -- Rhy.}
    \end{cxt}
  \end{center}
  Hence $\mathcal{L}=\{a \implies b\}$ and let
  $\mathcal{L}_1 = \mathcal{L}_2 = \emptyset$.  Then
  $\context_1\lor_g \context_2= \context^U$ and
  $E_1\lor E_2 = (\context^U,\emptyset)$.  However, unless we know
  that $\context_1\lor_g \context_2$ contains all objects from the
  domain (or at least one object for every distinct set of attributes
  that appears in the universe) we cannot use the context of
  counterexamples to obtain missing valid implications,
  cf. \cref{remark:validity_implications_examples_expert}:

  If we have that all objects are known we can improve the set of
  implications by combining it with the set of certainly valid
  implications from the combined incomplete context:
  \[
    (\context_1 \lor_g \context_2, \Cons{\Imp[\context_1 \lor_g
      \context_2] \cup \mathcal{L}_1 \cup \mathcal{L}_2 } )
  \]
  However, this is not true in general. Consider two experts
  $E_3=\expert{3}$ and $E_4=\expert{4}$ for the formal context
  $\context^U$ given by:
  \begin{center}
    \begin{cxt}
      \cxtName{$\context_3$} \cxtNichtKreuz{} \att{a} \att{b}
      \obj{xx}{Aquatics -- Swimming}
    \end{cxt}
    \begin{cxt}
      \cxtName{$\context_4$} \cxtNichtKreuz{} \att{a} \att{b}
      \obj{.x}{Badminton}
    \end{cxt}
  \end{center}
  Once more $\mathcal{L}=\{a\implies b\}$ and let
  $\mathcal{L}_3 = \mathcal{L}_4 = \emptyset$.  Then
  $\context_3 \lor_g \context_4$ is defined by
  \begin{center}
    \begin{cxt}
      \cxtName{$\context_3 \lor_g \context_4$} \cxtNichtKreuz{}
      \att{a} \att{b} \obj{xx}{Aquatics -- Swimming}
      \obj{.x}{Badminton}
    \end{cxt}
  \end{center}
  and $\Cons{ \mathcal{L}_1 \cup \mathcal{L}_2 }= \emptyset$.
  Therefore $E_3\lor E_4 = (\context_3 \lor_g \context_4, \emptyset)$.
  If we tried to use the certainly valid implications of the example
  context we would obtain
$$\Cons{\Imp[\context_3 \lor_g \context_4] \cup \mathcal{L}_3 \cup \mathcal{L}_4 } = \{\emptyset \implies b, a \implies b\} \nsubseteq \mathcal{L}.$$

One way to deal with this problem could be to introduce an incomplete
context $\context_{\q} = \GMWI[\q]$ where $I_{\q}(g,m)= \q $ for all
$g\in G$ and $m \in M$ and define
\[
  (\context_1 \lor_g \context_2, \Cons{\Imp[\context_1 \lor_g
    \context_2 \lor_g \context_{\q}] \cup \mathcal{L}_1 \cup
    \mathcal{L}_2 } ).
\]
Still, for all practical purposes where we are unable to determine all
of the objects belonging to a domain we cannot make use of the
examples to expand the set of known implications.

\end{rem}

\subsection{Expert Interaction and Collaboration Strategy}

Now that we have formalized \emph{expert knowledge} we need to
consider what it means to \emph{interact with an expert}, i.e., how an
expert can respond to a question.
\begin{defn}[expert interaction]
  Given an expert $E=\expert{E}\in\mathbf{E}$ from the set of all
  possible experts $\mathbf{E}$ for a universe $\context^U$, we view
  the answer given by the expert as a function
  \begin{align*}
    EI: \Impm \times \mathbf{E} \rightarrow \{(\text{true},\emptyset)\} &\cup\{\text{false}\} \times \{ \context \mid \context \text{ an incomplete context over } M\} \\
                                                                        &\cup \{\text{unknown}\} \times \powerset{M}
  \end{align*}
  where the input $(A\implies B, E) \in \Impm \times \mathbf{E}$ is
  considered as asking the expert ``Does $A\implies B$ hold in the
  universe?'' and the answer given is consistent with the true
  knowledge, i.e., example and implication knowledge, of the domain.
  This means that an expert can only accept an implication if it is
  actually true, i.e., in $\Imp[\context^U]$, and can only reject
  questions with real examples from the universe.
\end{defn}

This definition allows for experts to withhold knowledge.  For
example, because checking all examples an expert could think of might
take too long the expert only checks the first couple of examples that
come to mind.  For now we employ an expert (analog to the expert used
by Holzer, cf. \cite{holzer2001dissertation}) that answers as fully as
possible.  The interaction with such an expert is formalized by the
\emph{standard expert interaction}.  An expert with standard expert
interaction is thought to give answers that contain all that the
expert knows.

\begin{defn}[standard expert interaction]
  Recall the attribute exploration in
  \cref{subsec:attribute_exploration}.  Formally the expert
  $E= \expert{E}$ answers in the following way.
    $$EI_S(A\implies B, E) = \begin{cases}
      (\text{true}, \emptyset) &\tif A\implies B \in \Cons{\mathcal{L}_E}\\
      (\text{false},\context_C) &\tif \{g\in G_E \mid A \subseteq g^\Box \land (M\setminus g^\Diamond) \cap B \neq \emptyset\}=:C \neq \emptyset, \\&\text{ where } \context_C := \context_E|_C\\
      (\text{unknown}, Z) &\totherwise, \text{ where } Z:= B\setminus
      \langle A\rangle_{\mathcal{L}_E}
    \end{cases}$$
  \end{defn}
%
  Now, we adapt the idea of a consortium,
  cf. \cite{conf/iccs/HanikaZ18}, which is basically a fixed group of
  experts that have a given way of responding to an implicational
  question.  In our formalization the knowledge of an expert and the
  ways she can interact with the knowledge are separate. 

  Here, a \emph{group of experts} is a set of experts that are all
  experts for the same universe $\context^U$.  A \emph{group of experts} will
  usually be denoted by $\mathcal{E}=\{E_1,\ldots,E_n\}$.  For now, we
  also consider the group of experts that conduct an exploration fixed.
  Evaluating the implications of dynamically changing expert sets is a
  topic for future work.
  Note that every expert in a group of experts could have her own
  method of interaction with the expert knowledge.  At present, we
  assume that all experts can be interacted with in the same way.
  Allowing and combining different modes of interaction is a topic for
  future research.

  It easily follows that the knowledge of all experts is compatible:
  \begin{lem}\label{lem:expert_knowledge_of_a_group_of_experts_is_compatible}
    Let $\{E_1,\ldots, E_n\}$ be a group of experts for the domain
    $\context^U$ where $E_i = \expert{E_i}$ for $i\in \{1,\ldots,n\}$.
    Then the example knowledge of all experts has no conflicting
    information and all implications known by each expert are
    satisfiable implications for all experts in the pool.
  \end{lem}
  \begin{proof}
    That there is no conflicting information is a direct conclusion
    from
    \cref{cor:expert_example_knowledge_forms_a_lattice_when_adding_the_domain_context_and_empty_context}.
    Further, for every combination of $i$ and $j$ with
    $i,j\in\{1,\ldots,n\}$ we have that
    $Y_{ij}\coloneqq(\context_{E_i},\Cons{\mathcal{L}_{E_j}})$ is an
    expert for $\context^U$.  Since $Y_{ij}$ is an expert, using
    \cref{lem:expert_knowledge_is_compatible}, it follows that all
    implications in $\Cons{\mathcal{L}_{E_j}}$ are satisfiable in the
    context of examples $\context_{E_i}$.  Hence, the know
    implications of each expert are satisfiable in the known example
    contexts of all experts.
  \end{proof}
%

  Now we have to think about how a group of experts for a domain can
  be used to answer a question posed by the exploration algorithm.  In
  other words, we have to think about how a group of experts can
  cooperate.
%
%
%
  To this end we formalize the notion of a \emph{collaboration
    strategy}.
  \begin{defn}[collaboration strategy]
    A \emph{collaboration strategy} is an algorithm that, given a
    group of experts $\mathcal{E}:=\{E_1, \ldots E_n\}$ for the
    universe $\context^U=\GMI$, an expert interaction strategy $EI$
    and a question in form of an implication $A\implies B$
    ($A,B\subset M$), returns an answer that is consistent with the
    universe, i.e., that the algorithm only accepts valid implications
    and only rejects invalid implications with proper counterexamples.
    This means that the \emph{collaboration strategy} can be seen as a
    function
    \begin{align*}
      \phi:\Impm \times \{ \mathcal{E} \} \times \{EI\} \rightarrow &\{(\text{true},\emptyset)\}\\ &\cup
                                                                                                     \{\text{false}\} \times \{ \context \mid \context \text{ an incomplete context over }M\}\\ &\cup \{\text{unknown}\} \times \powerset{M}
    \end{align*}
    with the properties
    \begin{enumerate}
    \item
      $\phi(A\implies B,\ \mathcal{E},\ EI) = (\text{true},\emptyset)
      \Implies A\implies B \in \Imp[\context^U]$
    \item
      $\phi(A\implies B,\ \mathcal{E},\ EI) = (\text{false},\context)
      \Implies \forall g\in G\colon A \subseteq g^\Box \land
      (M\setminus g^\Diamond) \cap B \neq \emptyset$ \\where
      $\context = \GMWI \leq_g \context^U$
    \item
      $\phi(A\implies B,\ \mathcal{E},\ EI) = (\text{unknown},Z)
      \Implies A\implies (B\setminus Z) \in \Imp[\context^U] \tand
      Z\subseteq B$.
    \end{enumerate}
    
  \end{defn}

  A \emph{collaboration strategy} takes the role of an intermediary
  between the attribute exploration algorithm and the group of
  experts.  It takes the questions posed by the exploration and
  interacts with the experts to find an answer which it then reports
  back to the exploration algorithm,
  cf. \cref{fig:collab_exploration}.


  \subsection{Discussion of Collaboration Strategies}
  \label{subsec:discussion_of_collaboration_strategies}
  We first present different collaboration strategies.  To emphasize
  that our approach for a group of experts is a generalization and
  contains the classical attribute exploration, we begin with a
  strategy for the exploration with a single expert.
  \begin{strategy}[single expert]
    \label{strategy:single_expert}
    Given we only have one expert $E$ and the interaction strategy
    $EI$.  The canonical strategy is to relay the questions to the
    expert and the answers back to the attribute exploration without
    any modifications.  If we use the \emph{standard interaction
      strategy} $EI_S$ then
    \cref{fact:attribute_exploration_results_in_maximum_knowledge_w_r_t_expert_knowledge}
    guarantees that this results in the maximal information obtainable
    with respect to the expert's knowledge.  The \emph{single expert
      strategy} $\phi_{single}$ is defined by:
    $$
    \phi_{single}(A\implies B, \{E\}, EI) := EI(A\implies B, E).
    $$
    Clearly this strategy gives answers consistent with the domain.
  \end{strategy}

  We continue with two extreme cases of collaboration strategies,
  namely the \emph{ignorant strategy} and the \emph{maximum knowledge
    strategy}.  The former does not bother interacting with any of the
  experts and instead always responds with `unknown' and the latter
  asks the experts for everything they know before giving the best
  possible answer.
  \begin{strategy}[ignorant]
    \label{strategy:ignorant_strategy}
    Given a group of experts $\mathcal{E}=\{E_1,\ldots,E_n\}$ for a
    universe $\context^U$ and an interaction strategy $EI$.
    The \emph{ignorant strategy} $\phi_{ignorant}$ is defined by:
    $$
    \phi_{ignorant}(A\implies B, \mathcal{E}, EI):= (\text{unknown},
    B).
    $$
\end{strategy}
This strategy can be interpreted as modelling experts that are
unwilling to participate in the exploration.  Then the only valid
option to give answers consistent with the universe is to use a
strategy that always answers `unknown'.  Note that this strategy also
works if the set of experts $\mathcal{E}$ is empty.

The \emph{maximum knowledge strategy} is an easy strategy that allow
experts to obtain maximum knowledge about a domain.
\begin{strategy}[maximum knowledge]\label{strategy:maximum_knowledge}
  Given a group of experts $\mathcal{E}:=\{E_1,\ldots,E_n\}$,
  $n\geq 1$, for a universe $\context^U$.  We first ask all experts
  for all of their example knowledge and all of their implicational
  knowledge.  Then we define a new artificial expert that has the
  combined knowledge of all experts.
    $$
    E_{max} := (\bigvee_{i=1}^n \context_{E_i}, \Cons{\bigcup_{i=1}^n
      \mathcal{L}_{E_i}}) = \bigvee_{i=1}^n E_i
    $$
    The \emph{maximum knowledge strategy} $\phi_{max}$ can now be
    described as relaying the question to the artificial maximal
    expert using the standard expert interaction.
    $$
    \phi_{max}(A\implies B, \mathcal{E}, EI) :=\phi_{single}(A\implies
    B, \{E_{max}\}, EI_S) = EI_S(A\implies B, E_{max}).
    $$
\end{strategy}
This strategy can be interpreted as modelling the experts sitting
together and discussing every question posed by the exploration
algorithm until they find the best answer they can provide together.
\begin{thm}\label{thm:max_expert_strategy_is_optimal}
  Given a group of experts $\mathcal{E}:=\{E_1,\ldots,E_n\}$,
  $n\geq 1$, for a universe $\context^U$, the maximal expert
  $E_{max}= \bigvee_{i=1}^n E_i$ has the maximal knowledge of the
  group of experts.  Attribute exploration with the strategy
  $\phi_{max}$ results in the maximal obtainable knowledge from the
  experts $\mathcal{E}$.
\end{thm}
\begin{proof}
  The maximal expert represents the maximal knowledge of the group of
  experts by definition as their supremum in the set of possible
  experts for the universe
  $\{ E \mid E \text{ expert for } \context^U\}$.  The attribute
  exploration with a single expert results in the maximum of
  obtainable knowledge from the expert,
  cf. \cref{fact:attribute_exploration_results_in_maximum_knowledge_w_r_t_expert_knowledge}.
  Therefore $\phi_{max}$ results in the maximum of obtainable
  knowledge from the group of experts $\{E_1,\ldots,E_n\}$.
\end{proof}
Even though \cref{strategy:maximum_knowledge} is a valid collaboration
strategy it does not represent a realistic way of collaboration if the
group of experts becomes large.  Further, if we ignore the
interpretation of this strategy as sitting together and discussing the
questions, we realize that it is highly unlikely that an expert with
reasonably large knowledge of any domain can reproduce all facts about
it at once.  Nonetheless, this collaboration strategy serves as a
means to evaluate other collaboration strategies as the maximal expert
represents the maximum of obtainable knowledge from the group of
experts, cf. \cref{thm:max_expert_strategy_is_optimal}.

There are of course some more realistic collaboration strategies.  In
the following we present the \emph{broadcast
  strategy}~(\cref{strategy:broadcast}), the \emph{iterative
  strategy}~(\cref{strategy:iterative_asking_experts}) and the
\emph{random selection strategy}~(\cref{strategy:random_selection}),
each of which models one way of asking a group of experts.
The idea of the \emph{broadcast strategy} is to ask all experts in a
group and combine the answers to form the best collaborative answer
possible.  This requires the strategy to interact with every expert
for every question but the experts can be asked independently from
each other and thus this can be done simultaneously.  The time
required to answer a question using this strategy is the longest time
it takes for any of the experts to answer.
\begin{strategy}[broadcast]
  \label{strategy:broadcast}
  Let $\{E_1,\ldots,E_n\}$ be a group of experts for the universe
  $\context^U$ with the expert interaction $EI$ for all experts.
  Given a question `Does $A\implies B$ hold in the universe?'  the
  \emph{broadcast strategy} asks all experts at once and collects all
  the answers, i.e., the results of all interactions
  $EI(A\implies B, E_i)$.  Then it combines the answers to form the
  best possible response, i.e., the \emph{broadcast strategy} combines
  the attributes known to follow from the premise and it combines all
  counterexamples.
\end{strategy}

A more sensible approach in terms of the number of required expert
interactions is to order the experts randomly for every question and
relay the question to one expert at a time until either one expert
accepted or rejected the attribute implication or all experts have
been consulted, cf. \cref{strategy:iterative_asking_experts}.
Basically this strategy increases the average time required to answer
a question in order to reduce the average number of expert
interactions required.
That way the amount of work required from every expert is reduced and
we presume that the result in terms of how much knowledge is obtained
stays about the same.  Note that the amount of example knowledge might
be reduced since in general not all known counterexamples to an
implication will be found if not all experts are consulted.  Further
note that the implication knowledge should not be impacted by the
iterative approach, since no implication knowledge of not consulted
experts is lost if an implication is accepted.  If we know about what
the experts know it seems reasonable to order the experts in an
optimal way for every question, again this is a subject for future
research.
\begin{strategy}[iterative]
  \label{strategy:iterative_asking_experts}
  Let $\{E_1,\ldots,E_n\}$ be a group of experts for the universe
  $\context^U$ with the standard expert interaction $EI_S$ for all
  experts.  In \cref{alg:iterative_collaboration_strategy} we present
  the \emph{iterative strategy} $\phi_{iter}$ that asks questions
  iteratively to reduce the number of interactions needed per question
  when exploring the domain.
        
        \begin{algorithm}[t]
          \SetKwComment{Comment}{}{} \SetKw{Kwin}{in}
          \DontPrintSemicolon \SetAlgoLined \KwIn{$A\implies B$,
            $\{E_1,\ldots,E_n\}$} \KwOut{Collaborative Answer}
          $Y := \emptyset$ \tcc*{collect the set of attributes known
            to follow from $A$} \For(\tcc*[f]{iterate the group of
            experts}){$E$ \Kwin $\{E_1,\ldots,E_n\}$}{
            $R := EI_S(A\implies B, E)$ \tcc*{ask the expert and
              collect the response}
            \lIf{$R=(\ttrue,\emptyset)$}{\KwRet$(\ttrue,\emptyset)$}
            \lIf{$R=(\tfalse,
              \context_C)$}{\KwRet$(\tfalse, \context_C)$}
            \lIf{$R=(\tunknown, Z)$}{$Y:=Y\cup (B\setminus Z)$} }
          \If(\tcc*[f]{if it is known conjointly that $B$ follows from
            $A$}){$Y=B$}{\KwRet$(\ttrue,\emptyset)$}
          \KwRet$(\tunknown, B\setminus Y)$
          \caption{Collaboration Strategy $\phi_{iter}$ using
            iterative questions}
          \label{alg:iterative_collaboration_strategy}
        \end{algorithm}
    
        The answers produced by $\phi_{iter}$ are consistent with the
        domain: An implication is only rejected if an expert knows a
        counterexample, it is accepted if either an expert accepted it
        or the group of experts knows that the conclusion follows from
        the premise.  Otherwise the response is `unknown'.
    
      \end{strategy}
      \begin{rem}
        An extensive example of a collaborative exploration using the
        \emph{iterative strategy} can be found in the appendix (see
        \cref{example:appendix}).  Subject of the exploration is a
        subset of the attributes of the \emph{Disciplines of the
          Summer Olympic Games 2020} context from the introduction
        (see \cref{sec:introduction}).  The complete context can also
        be found in the
        \hyperref[appendix:full_example_context]{appendix}.
      \end{rem}

      Another simple approach to adapt the attribute exploration to a
      group of experts is to randomly select a new expert to ask each
      time the exploration algorithm poses a question.  The
      \emph{random selection strategy} describes this form of
      collaboration where every expert is equally likely to be
      selected.  This strategy is useful to balance the amount of
      interaction required from each expert.  It further reduces the
      average time needed to answer a question compared to the
      \emph{broadcast strategy} because here the expert that takes the
      longest to answer is not always asked.  However, in general the
      \emph{random selection strategy} is strictly worse than the
      \emph{broadcast} in terms of obtained knowledge.  An exception
      is the case where all experts have the same knowledge.  Then the
      \emph{random selection strategy} results in the same amount of
      obtained knowledge but, compared to the \emph{broadcast
        strategy}, with fewer expert interactions and after a shorter
      period of time.
      \begin{strategy}[random selection]
        \label{strategy:random_selection}
        Given a group of experts $\mathcal{E}:=\{E_1,\ldots,E_n\}$,
        $n\geq 1$, for a universe $\context^U$ and an interaction
        strategy $EI$.  Consider the \emph{random selection strategy}
        $\phi_{rand}$ that randomly asks one of the experts and simply
        gives the answer as response
        \[
          \phi_{rand}(A\implies B, \mathcal{E}, EI) :=
          \phi_{single}(A\implies B, \{E_i\}, EI) \text{ where } E_i
          \sim \mathrm{Uniform}(\{E_1,\ldots,E_n\}).
        \]
        $\mathrm{Uniform}(\{E_1,\ldots,E_n\})$ denotes randomly
        picking an expert with equal probability.
      \end{strategy}
      Note that instead of the $\mathrm{Uniform}$ distribution we
      could also employ other discrete probability distributions to
      assign weights to the experts.

      \begin{rem}
        \label{rem:accepting_implications_with_the_ususal_expert_can_be_complex}
        The standard expert interaction exhibits an interaction
        complexity issue when multiple experts together know that an
        implication is valid but none of the experts alone knows this.
        \cref{fig:context_example_inferring_implications_is_hard}
        gives an example of a context where two experts can have
        adversely distributed implication knowledge.  When using the
        standard expert interaction it can happen that many
        interactions are required to confirm a valid implication.

        The set of valid implications in $\context$
        (cf. \cref{fig:context_example_inferring_implications_is_hard})
        is $\Imp[\context] = \{a\implies bc,\ b\implies c\}$.
        Consider two experts $\{E_1, E_2\}\eqqcolon \mathcal{E}$ with
        implicational knowledge $\mathcal{L}_1 = \{ a \implies b\}$
        and $\mathcal{L}_2 = \{ b \implies c\}$ and the standard
        expert interaction $EI_S$.  If the question `Does $a$ imply
        $c$?' is posed none of the experts could agree and even if
        they were allowed to work together using the standard expert
        interaction none of them could report anything since $EI_S$
        only allows to report parts of the conclusion that are known
        to follow.
    
        In comparison if the question `Does $a$ imply $bc$?' is posed
        to both experts, at least expert $E_1$ would report that $b$
        is known to follow from the premise.  To also obtain that $c$
        follows, the expert $E_2$ needs to be consulted a second time
        about the question whether the implication $ab\implies c$ is
        true.  Then $E_2$ would accept this implication to be valid
        and therefore confirm that the implication $a\implies c$ is
        valid.
    
\begin{figure}[t]
  \centering
  \begin{cxt}
    \cxtName{$\context$} \cxtNichtKreuz{} \att{a} \att{b} \att{c}
    \obj{xxx}{1} \obj{.xx}{2} \obj{..x}{3} \obj{...}{4}
  \end{cxt}
  \caption{An example of a context where implicational knowledge can
    be adversely distributed between experts.}
  \label{fig:context_example_inferring_implications_is_hard}
\end{figure}

Clearly this poses a problem for defining collaboration strategies.
The example shows that there are cases when an implication can be
accepted by a group of experts $\mathcal{E}$ even though no single
expert alone can accept it but that this requires potentially repeated
questioning of the experts for every single question posed by the
attribute exploration.  The first part of the example shows that there
exist implications that cannot be accepted using the standard expert
interaction even though they would be accepted by the maximal expert
$E_1 \lor E_2$.  Furthermore, it usually requires every expert to be
consulted if nothing about the experts knowledge is known.
    
A quick generalization of the previous example shows that there are
cases where the number of required interactions with the experts
$\mathcal{E}$ is at least $|\mathcal{E}| \cdot
(|M|-1)$$\colon$ Let
$\context$ be a context with
$M=\{a_0,\ldots,a_n\}$ where the implications $a_0 \implies a_1\ldots
a_n$ and $a_i \implies a_{i+1}$ with
$i\in\{0,\ldots,n-1\}$ are valid.  Clearly there exist contexts where
these implications are valid, for example, the empty context on the
attribute set
$M$ where every implication is valid or a sufficiently large ordinal
scale, cf. \cite{GanterWille1999}.  Consider a group of experts
$\mathcal{E} =\{E_1,\ldots, E_{m}\}$, $m\geq n$, for
$\context$ such that the expert $E_i$ knows the implication $a_{i-1}
\implies a_{i}$ for $1\leq i\leq
n$ and has no implicational knowledge for $n< i \leq m$.
To answer the question `Does $a_0$ imply $a_1\ldots
a_n$?' every expert needs to be consulted
$n$ times.  The first iteration results in $a_0\implies
a_1$ which is turned into the question `Does $a_0a_1$ imply $a_1\ldots
a_n$?' for the second iteration and so on.  After
$n$ iterations it is finally known that the implication $a_0\implies
a_1\ldots a_n$ is valid.  In each iteration, all
$m$ experts need to be consulted to assure that no information was
missed.  Hence the number of interactions needed to accept an
implication can reach
$|\mathcal{E}|\cdot(|M|-1)$.  Even in the best case when asking
iteratively and always the right expert first it would require
$n-1$ interactions to accept the implication.
    
Note that this is not only a problem due to the definition of the
standard expert interaction but a fundamental problem that arises when
exploring distributed implication knowledge.  From a theoretical point
we could simply redefine the notion of expert interaction and let
experts respond with all the implicational knowledge they have.
However, from a real world perspective (and sticking to our Olympic
Disciplines example from \cref{sec:introduction}) it would be like
asking ``Do all disciplines that have \emph{female-only events} also
have \emph{male-only events}?''  and to receive an answer like ``No,
but if a discipline holds \emph{more than ten events} it was
\emph{part of at least eight Olympic Games}'' on the off-chance that
those facts are somehow related.  Clearly such an approach does not
make much sense.
    
\end{rem}
\begin{cor}
  Given a universe $\context^U=\GMI$ with $|M|=n+1$.  Let
  $\{E_1,\ldots,E_m\}$ be a group of experts for $\context^U$ and the
  interaction with the experts is the standard expert interaction
  $EI_S$.  The worst case number of interactions required to accept a
  single valid implication is at least $m\times n$.
\end{cor}
\begin{proof}
  This directly follows from
  \cref{rem:accepting_implications_with_the_ususal_expert_can_be_complex}.
\end{proof}

One possible solution to mitigate this problem is to allow the
collaboration strategies to report with more than just `yes', `no' or
`unknown' back to the attribute exploration.  Then the collaboration
strategy could collect the known implications found during the
interactions with the experts and add the implications to the reply.
This might reduce repetitive interactions by preventing some questions
where the answer can then be inferred directly by the exploration
algorithm.  However, this would require modifying the attribute
exploration to allow using such additional information.
\subsection{Comparing and Evaluating Collaboration Strategies}
\label{subsec:comparing-and-evaluating-collab-strategies}
One important topic that we have touched before but not explicitly
discussed yet is the question of how to evaluate and compare
collaboration strategies.  We have already mentioned three criteria
that seem important: The portion of knowledge that is obtained,
the (average) time needed and the experts effort required to
obtain it.

First of all, note that the result of a collaborative exploration is
very much dependant on the group of experts.  Therefore, the maximal
expert and the maximal collaboration strategy are useful tools to
evaluate the portion of obtained knowledge relative to the possible
maximum.

The result of an attribute exploration consists of the set of known
valid implications and the incomplete context of counterexamples.
Hence, the result of an exploration is an element of the product
lattice of implications and examples and can be compared in the same
way as expert knowledge.  However, this also means that different
exploration results may be incomparable in this order relation.  A
path to circumvent this problem could be to define some metric on the
elements of the product lattice to capture the relative knowledge in a
single number and make it easily comparable.  Though, how to define
such a metric is not obvious to the authors.

Another difficulty inherent in the definition of collaboration
strategy as algorithm is that random algorithms can be used, which
implies that the answer to a question is not guaranteed to be
deterministic.

An example for such a strategy is \cref{strategy:random_selection}. It
can be used to show that strategies which incorporate random elements
are even more difficult to evaluate:

\begin{exmpl}
  \label{example:random_strategies_cause_problems}
  Assume that we want to explore a universe $\context^U$ with two
  experts for the domain $E_1=\expert{1}$ and $E_2=\expert{2}$
  utilizing the collaboration strategy $\phi_{rand}$,
  cf. \cref{strategy:random_selection}.  Let the expert $E_1$ be
  all-knowing, i.e., $\context_{1} = \context^U$ and
  $\mathcal{L}_1 = \Imp[\context^U]$, and let the expert $E_2$ be
  completely ignorant, i.e.,
  $\context_{2} =(\emptyset,M,\values,I_\emptyset)$ and
  $\mathcal{L}_2 = \emptyset$.
    
  If we explore the universe using the \emph{random selection
    strategy} then the exploration result can vary extremely depending
  on the randomness.  It could be that no knowledge at all was
  discovered if the strategy always selected the expert $E_2$, it
  could be a completely explored domain if the strategy always
  selected the expert $E_1$ or it could be anything in between.
\end{exmpl}

Another important consideration when evaluating a collaboration
strategy is how much effort by the group of experts is needed.  This
can be measured, for example, by counting the required interactions
per expert and comparing based on total, average, maximum or other
suitable metrics.  Note that minimizing expert effort alone is
problematic.  An optimal strategy in this regard is the \emph{ignorant
  strategy}, cf. \cref{strategy:ignorant_strategy}, which does not
bother interacting with the experts, always answers `unknown' and
clearly is not what we want.

Now we give a brief ranking of the three more realistic collaboration
strategies from \cref{subsec:discussion_of_collaboration_strategies}
(\emph{broadcast}, \emph{iterative} and \emph{random selection}).  We
compare them based on \emph{knowledge obtained}~(K), \emph{time
  needed}~(T) and \emph{number of expert interactions}~(I):
\begin{itemize}
\item[K:]\label{itemize_K} We presume that the \emph{broadcast} and
  the \emph{iterative} strategy result in about the same obtained
  knowledge
  (cf. \cref{subsec:discussion_of_collaboration_strategies}). However,
  the \emph{broadcast} strategy will result in a more extensive set of
  counterexamples. The \emph{random selection} strategy is difficult
  to compare to the two other strategies because of its randomness
  (cf. \cref{example:random_strategies_cause_problems}), though we
  expect this strategy to obtain far less knowledge in general.
    
\item[T:] The \emph{random selection} strategy takes the least time,
  followed by the \emph{broadcast} strategy and the worst in terms of
  time needed is the \emph{iterative} strategy due to its sequential
  expert interactions.
    
\item[I:] The \emph{random selection} strategy requires the least
  number of expert interactions (one per question), the
  \emph{broadcast} strategy requires the most number of interactions
  (one per expert per question) and the \emph{iterative} strategy
  requires a number of interactions in between depending on the order
  in which the experts are asked and how knowledgable the experts are.
\end{itemize}

It has become apparent that comparing collaboration strategies is a
complex task where further research is required.  Defining what
characterizes a `good' collaboration strategy is hard.  As a rule of
thumb it balances the knowledge obtained, the time needed and the
effort required from the experts.

\section{Conclusions and Outlook}
\label{sec:conclusion_and_outlook}
We have extended the theory of attribute exploration for incomplete
knowledge to work in a setting of multiple experts with incomplete
knowledge of a domain.  To this end we have formalized expert
knowledge as a tuple of (possibly incomplete) examples and valid
implications and formalized a notion of interaction with expert
knowledge.  Further we have defined a collaboration strategy as an
algorithm that takes an implicational question and a group of experts
as input and returns an answer that fits the scheme required by the
attribute exploration for incomplete knowledge in
\cite{holzer2001dissertation}.  Orders on incomplete contexts and
expert knowledge have been introduced to facilitate comparability of
the results of attribute explorations by multiple experts.  Some
collaboration strategies and ways to compare such strategies in
general have been discussed.  Numerous questions and avenues for
further research have been identified.

In particular, we will develop further characterizations of `good'
collaboration strategies.  One problem that should be tackled in the
future
is 
to find a metric which allows for an easy comparison of the knowledge
discovered as result of attribute exploration with some collaboration
strategy.  Basically this implies defining a metric on the lattice of
possible exploration results.  Other avenues for future research could
deal with: considering changing sets of experts, different modes of
interaction, more specifically defined experts (such as the experts in
\cite{conf/iccs/HanikaZ18}) or assuming knowledge about the knowledge
of experts to reduce the amount of interactions required.


\sloppy \printbibliography

\newpage
\section{Appendix}
\label{sec:appendix}
\subsection{Disciplines of the Summer Olympic Games 2020 Context}
\label{appendix:full_example_context}
\setlength{\LTleft}{0pt} \setlength{\LTright}{0pt}

\begin{cxtl}
  \cxtName{} \cxtNichtKreuz{}
    
  \atr{$\geq$ 5 events} \atr{$\geq$ 10 events} \atr{$\geq$ 20 events}
  \atr{athletic sport} \atr{ball game} \atr{combat sport} \atr{female
    only events} \atr{has paralympic equivalent} \atr{individual
    competition} \atr{indoor events} \atr{male only events} \atr{mixed
    events} \atr{open events} \atr{outdoor events} \atr{part of $\geq$
    8 olympics} \atr{part of $\geq$ 16 olympics} \atr{part of $\geq$
    24 olympics} \atr{team competition} \atr{water sport}
  \obj{...x..x..x....x..xx}{Aquatics -- Artistic Swimming}
  \obj{x..x..x.xxx...xxxxx}{Aquatics – Diving}
  \obj{...x..x.x.x..x....x}{Aquatics – Marathon Swimming}
  \obj{xxxx..xxxxxx..xxxxx}{Aquatics – Swimming}
  \obj{...xx.x..xx...xxxxx}{Aquatics – Water Polo}
  \obj{x..x.xxxx.xx.xxx.x.}{Archery}
  \obj{xxxx..xxx.xx.xxxxx.}{Athletics}
  \obj{x..x..xxxxxx..x..x.}{Badminton}
  \obj{...xx.x...x..x...x.}{Baseball/Softball}
  \obj{...xx.x..xx......x.}{Basketball – 3x3}
  \obj{...xx.xx.xx...xx.x.}{Basketball – Basketball}
  \obj{xx.x.xx.xxx...xxx..}{Boxing} \obj{...x..x.x.x..xx...x}{Canoe –
    Slalom} \obj{xx.x..xxx.x..xxx.xx}{Canoe – Sprint}
  \obj{...x..x.x.x..x.....}{Cycling – BMX Freestyle}
  \obj{...x..x.x.x..x.....}{Cycling – BMX Racing}
  \obj{...x..x.x.x..x.....}{Cycling – Mountain Bike}
  \obj{...x..xxx.x..xxxx..}{Cycling – Road}
  \obj{xx.x..xxxxx...xxxx.}{Cycling – Track}
  \obj{.......xx...xxxxxx.}{Equestrian – Dressage}
  \obj{........x...xxxxxx.}{Equestrian – Eventing}
  \obj{........x...xxxxxx.}{Equestrian – Jumping}
  \obj{xx...xxxxxx...xxxx.}{Fencing}
  \obj{...xx.xx..x..xxxxx.}{Football} \obj{....x.x.x.x..x.....}{Golf}
  \obj{xx.x..x.xxx...xxxx.}{Gymnastics – Artistic}
  \obj{...x..x.xx....x..x.}{Gymnastics – Rhythmic}
  \obj{...x..x.xxx........}{Gymnastics – Trampoline}
  \obj{...xx.x..xx...x..x.}{Handball}
  \obj{...xx.x...x..xxxxx.}{Hockey} \obj{xx.x.xxxxxxx..x..x.}{Judo}
  \obj{...x.xx.xxx........}{Karate – Kata}
  \obj{x..x.xx.xxx........}{Karate – Kumite}
  \obj{...x.xx.x.x..xxxx.x}{Modern Pentathlon}
  \obj{xx.x..xxx.x..xxxxxx}{Rowing} \obj{...xx.xx..x..x...x.}{Rugby –
    Rugby Sevens} \obj{xx....x.x.xx.xxxxxx}{Sailing}
  \obj{xx...xxxx.xx.xxxxx.}{Shooting}
  \obj{...x..x.x.x..x.....}{Skateboarding}
  \obj{...x..x.xxx........}{Sport Climbing}
  \obj{...x..x.x.x..x....x}{Surfing} \obj{x..xx.xxxxxx..x..x.}{Table
    Tennis} \obj{x..x.xxxxxx........}{Taekwondo}
  \obj{x..xx.xxx.xx.xxx.x.}{Tennis}
  \obj{...x..x.x.xx.x...xx}{Triathlon}
  \obj{...xx.x...x..x...x.}{Volleyball – Beach Volleyball}
  \obj{...xx.xx.xx...x..x.}{Volleyball – Volleyball}
  \obj{xx.x..xxxxx...xxx..}{Weightlifting}
  \obj{xx.x.xx.xxx...xxx..}{Wrestling – Freestyle}
  \obj{x..x.x..xxx...xxx..}{Wrestling – Greco Roman}
\end{cxtl}
The information for the \emph{Disciplines of the Summer Olympic Games
  2020} context was obtained from \url{https://tokyo2020.org/},
\url{https://www.olympic.org/tokyo-2020} and
\url{https://en.wikipedia.org/wiki/Olympic_sports}.

Note that the \emph{number of events} ($\geq$ 5, $\geq$ 10 and $\geq$
20) and the\emph{ number of Olympics that a discipline was of} ($\geq$
8, $\geq$ 16 and $\geq$ 24) are ordinally scaled attributes,
cf. \cite{GanterWille1999}.

\vspace{5ex}
\subsection{Example of a collaborative exploration with three experts}
\captionsetup[figure]{labelformat=empty}
\begin{example}\label{example:appendix}
  In the following we give an example of attribute exploration with
  multiple experts using the \emph{iterative collaboration strategy}
  (\cref{strategy:iterative_asking_experts}).  We explore a subset of
  the attributes of the \emph{Olympic Disciplines 2020},
  cf. \cref{appendix:full_example_context}, namely the attributes
  \emph{$\geq$ 5 events}, \emph{$\geq$ 10 events}, \emph{female only
    events}, \emph{male only events} and \emph{part of $\geq$ 8
    olympics}.
  The collaboration strategy makes use of three (ficticious) experts.
    
  The \textbf{first expert} $E_1=\expert{1}$ prefers Olympic
  Disciplines with a long tradition in the Olympic Games.  She knows
  that \emph{all disciplines with more than ten events also have more
    than five events and are part of at least eight olympics}, i.e.,
  \emph{$\mathcal{L}_1 = \{\{\geq \text{ 10 events}\} \rightarrow
    \{\geq \text{ 5 events}, \text{ part of }\geq \text{ 8
      olympics}\}\}$}.
    
  The \textbf{second expert} $E_2=\expert{2}$ is a fan of \emph{water
    sport} and likes to watch \emph{mixed events}.  She knows that
  \emph{all discipline with more than five events also have male only
    events}, i.e.,
  \emph{$\mathcal{L}_2 = \{ \{\geq\text{ 5 events }\}\rightarrow
    \{\text{ male only events}\}\}$}.
    
  The \textbf{third expert} $E_3=\expert{3}$ likes combat sports.  She
  only knows that \emph{all disciplines with more than ten events also
    have more than five events}, i.e.,
  \emph{$\mathcal{L}_3 = \{\{\geq \text{ 10 events}\} \rightarrow
    \{\geq\text{ 5 events}\}\}$}.
    
  The following contexts represent the example knowledge of the three
  experts:
    %
    
    %
    %
    %
    
    \begin{center}
      \footnotesize
        
        \begin{minipage}{0.49\textwidth}
          \centering \textbf{Expert 1} \vspace{1ex}
            
            \begin{cxt}
              \cxtName{$\context_1$} \cxtNichtKreuz{} \atr{$\geq$ 5
                events} \atr{$\geq$ 10 events} \atr{female only
                events} \atr{male only events} \atr{part of $\geq$ 8
                olympics}
              \obj{x.xxx}{Aq. -- Diving} \obj{xxxxx}{Aq. -- Swimming}
              \obj{..xxx}{Aq. -- Water Polo} \obj{xxxxx}{Athletics}
              \obj{xxxxx}{Boxing} \obj{..xxx}{Cycling -- Road}
              \obj{xxxxx}{Cycling -- Track} \obj{....x}{Equestrian --
                Dressage} \obj{....x}{Equestrian -- Eventing}
              \obj{....x}{Equestrian -- Jumping} \obj{xxxxx}{Fencing}
              \obj{..xxx}{Football} \obj{xxxxx}{Gymnastics --
                Artistic} \obj{..xxx}{Hockey} \obj{..xxx}{Modern
                Pentathlon} \obj{xxxxx}{Rowing} \obj{xxxxx}{Sailing}
              \obj{xxxxx}{Shooting} \obj{xxxxx}{Weightlifting}
              \obj{xxxxx}{Wrestling -- Freestyle}
              \obj{x..xx}{Wrestling -- Greco Roman}
            \end{cxt}
            
            \vspace{1ex}
            {$\mathcal{L}_1 = \{\{\geq \text{ 10 events}\} \rightarrow
              \{\geq \text{ 5 events}, \text{ part of }\geq \text{ 8
                olympics}\}\}$}
          \end{minipage}
          \begin{minipage}{0.49\textwidth}
            \centering \textbf{Expert 2} \vspace{1ex}
            
            \begin{cxt}
              \cxtName{$\context_2$} \cxtNichtKreuz{} \atr{$\geq$ 5
                events} \atr{$\geq$ 10 events} \atr{female only
                events} \atr{male only events} \atr{part of $\geq$ 8
                olympics}
              \obj{..x.x}{Aq. -- Artistic Swimming} \obj{x.xxx}{Aq. --
                Diving} \obj{..xx.}{Aq. -- Marathon Swimming}
              \obj{xxxxx}{Aq. -- Swimming} \obj{..xxx}{Aq. -- Water
                Polo}
                
              \obj{x.xxx}{Archery} \obj{xxxxx}{Athletics}
              \obj{x.xxx}{Badminton} \obj{..xxx}{Canoe -- Slalom}
              \obj{xxxxx}{Canoe -- Sprint} \obj{xxxxx}{Judo}
              \obj{..xxx}{Modern Pentathlon} \obj{xxxxx}{Rowing}
              \obj{xxxxx}{Sailing} \obj{xxxxx}{Shooting}
              \obj{..xx.}{Surfing} \obj{x.xxx}{Table Tennis}
              \obj{x.xxx}{Tennis} \obj{..xx.}{Triathlon}
            \end{cxt}
            
            \vspace{1ex}
            {$\mathcal{L}_2 = \{ \{\geq\text{ 5 events}\}\rightarrow
              \{\text{male only events}\}\}$}
          \end{minipage}
        
        \begin{minipage}{0.49\textwidth}
          \centering \textbf{Expert 3} \vspace{1ex}
            
            \begin{cxt}
              \cxtName{$\context_3$} \cxtNichtKreuz{} \atr{$\geq$ 5
                events} \atr{$\geq$ 10 events} \atr{female only
                events} \atr{male only events} \atr{part of $\geq$ 8
                olympics} \obj{x.xxx}{Archery} \obj{xxxxx}{Boxing}
              \obj{xxxxx}{Fencing} \obj{xxxxx}{Judo}
              \obj{..xx.}{Karate -- Kata} \obj{x.xx.}{Karate --
                Kumite} \obj{..xxx}{Modern Pentathlon}
              \obj{xxxxx}{Shooting} \obj{x.xx.}{Taekwondo}
              \obj{xxxxx}{Wrestling -- Freestyle}
              \obj{x..xx}{Wrestling -- Greco Roman}
            \end{cxt}
            
            \vspace{1ex}
            {$\mathcal{L}_3 = \{\{\geq \text{ 10 events}\} \rightarrow
              \{\geq\text{ 5 events}\}\}$}
          \end{minipage}
          \vspace{12ex}
        \end{center}

        For the purpose of this example the order in which the experts
        are asked is always the same: First $E_1$, then $E_2$ and last
        $E_3$.  We list the questions posed by the attribute
        exploration algorithm, all interactions with the experts and
        the relevant parts taking place in the collaboration strategy.
    
        \
        \\
        \noindent\textbf{Question 1} posed by the attribute
        exploration: \newline ``Do all Disciplines have more than five
        events, more than ten events, female only events, male only
        events and have been part of at least eight Olympic Games?''.
    
        \
        \\
        The corresponding short implicational form of this question
        is:
    $$ 
    {\footnotesize \emptyset \rightarrow \{\geq \text{5
        events},\geq\text{10 events},\text{female only
        events},\text{male only events},\text{part of }\geq \text{8
        olympics}\} \ ?  }
    $$
    From now on we use the short form of the questions to improve
    readability.
    
    \
    \\
    The \emph{collaboration strategy} then poses \textbf{Question 1}
    to the Experts: \newline \textbf{Interaction with $E_1$}: \newline
    The expert knows this to be false and responds with
    $(\text{false }, \context_{Q1})$.
    
    \begin{figure}
      \setlength{\belowcaptionskip}{-25pt}
      \begin{minipage}{0.49\textwidth}
        \centering \textbf{Counterexample Question 1} \vspace{1ex}
            
            \begin{cxt}
              \cxtName{$\context_{Q1}$} \cxtNichtKreuz{} \atr{$\geq$ 5
                events} \atr{$\geq$ 10 events} \atr{female only
                events} \atr{male only events} \atr{part of $\geq$ 8
                olympics} \obj{x.xxx}{Aq. – Diving} \obj{..xxx}{Aq. –
                Water Polo} \obj{..xxx}{Cycling – Road}
              \obj{....x}{Equestrian – Dressage}
              \obj{....x}{Equestrian – Eventing}
              \obj{....x}{Equestrian – Jumping} \obj{..xxx}{Football}
              \obj{..xxx}{Hockey} \obj{..xxx}{Modern Pentathlon}
              \obj{x..xx}{Wrestling – Greco Roman}
            \end{cxt}
            
            \vspace{1ex}
          \end{minipage}
          \begin{minipage}{0.49\textwidth}
            \centering \textbf{Counterexample Question 2} \vspace{1ex}
            
            \begin{cxt}
              \cxtName{$\context_{Q2}$} \cxtNichtKreuz{} \atr{$\geq$ 5
                events} \atr{$\geq$ 10 events} \atr{female only
                events} \atr{male only events} \atr{part of $\geq$ 8
                olympics} \obj{..xx.}{Aq. – Marathon Swimming}
              \obj{..xx.}{Surfing} \obj{..xx.}{Triathlon}
            \end{cxt}
            
            \vspace{1ex}
          \end{minipage}
          \caption{}
        \end{figure}
    
        \
        \\
        The \emph{collaboration strategy} returns the context of
        counterexamples provided by $E_1$ to the attribute
        exploration.

    %
    %
    %

        \
        \\
        \noindent\textbf{Question 2}:
        $\emptyset \rightarrow \{\text{part of }\geq \text{ 8
          olympics}\} \ ?$
    
    \noindent\textbf{Interaction with $E_1$}: \newline
    This is {unknown} to $E_1$ and she responds with
    $(\text{unknown }, \{\text{part of }\geq \text{ 8 olympics}\} )$.
    
    \noindent\textbf{Interaction with $E_2$}: \newline
    The expert knows this to be false and responds with
    $(\text{false }, \context_{Q2})$.
    
    \
    \\
    \noindent\textbf{Question 3}:
    $\{\geq \text{ 10 events}\} \rightarrow \{\geq \text{ 5 events},
    \text{ female only events}, \text{ male only events},\\ \text{part
      of } \geq \text{ 8 olympics}\}\ ?$

    \noindent\textbf{Interaction with $E_1$}: \newline
    This is {unknown} to $E_1$ and she responds with \newline
    $(\text{unknown }, \{\text{female only events}, \text{male only
      events}\})$.
    
    \noindent\textbf{Interaction with $E_2$}: \newline
    This is {unknown} to $E_2$ and she responds with \newline
    $(\text{unknown}, \{\text{female only events} \geq \text{5
      events},\text{male only events},\text{part of} \geq \text{8
      olympics}\})$.
    
    \noindent\textbf{Interaction with $E_3$}: \newline
    This is {unknown} to $E_3$ and she responds with \newline
    $(\text{unknown }, \{\text{female only events}, \text{male only
      events}, \text{part of } \geq \text{ 8 olympics}\})$.

    \
    \\
    \noindent Here the \emph{iterative strategy} collected the set of
    attributes known to follow and replies with
    $(\text{unknown }, \{\text{female only events}, \text{male only
      events}\})$.  The attribute exploration introduces two
    fictitious counterexamples, one for each of the unknown
    attributes.
    
    \
    \\
    \noindent\textbf{Question 4}:
    $\{\geq \text{ 10 events}\} \rightarrow \{\geq \text{ 5 events},
    \text{ part of } \geq \text{ 8 olympics}\}\ ?$
    
    \noindent\textbf{Interaction with $E_1$}: 
    The expert knows this to be true and responds with
    $(\text{true }, \emptyset)$.

    \
    \\
    \\
    \noindent\textbf{Question 5}:
    $\{\geq \text{ 5 events}\} \rightarrow \{\text{male only events},
    \text{ part of } \geq \text{ 8 olympics}\}\ ?$
    
    \noindent\textbf{Interaction with $E_1$}: \newline
    This is {unknown} to $E_1$ and she responds with \newline
    $(\text{unknown }, \{\text{male only events}, \text{ part of }
    \geq \text{ 8 olympics}\})$.
    
    \noindent\textbf{Interaction with $E_2$}: \newline
    This is {unknown} to $E_2$ and she responds with
    $(\text{unknown }, \{\text{part of } \geq \text{ 8 olympics}\})$.
    
    \noindent\textbf{Interaction with $E_3$}: \newline
    The expert knows this to be false and responds with
    $(\text{false }, \context_{Q5})$.

    \
    \\
    \noindent\textbf{Question 6}:
    $\{\geq \text{ 5 events}\} \rightarrow \{\text{male only
      events}\}\ ?$
    
    \noindent\textbf{Interaction with $E_1$}: \newline
    This is {unknown} to $E_1$ and she responds with
    $(\text{unknown }, \{\text{male only events}\})$.
    
    \noindent\textbf{Interaction with $E_2$}: \newline
    The expert knows this to be true and responds with
    $(\text{true }, \emptyset)$.

    \
    \\
    \noindent\textbf{Question 7}:
    $\{\geq \text{ 5 events}, \geq \text{ 10 events}, \text{male only
      events}, \text{part of } \geq \text{ 8 olympics}\} \rightarrow
    \{\text{female only events}\}\ ?$
    
    \noindent\textbf{Interaction with $E_1$}: \newline
    This is {unknown} to $E_1$ and she responds with
    $(\text{unknown }, \{\text{female only events}\})$.
    
    \noindent\textbf{Interaction with $E_2$}: \newline
    This is {unknown} to $E_2$ and she responds with
    $(\text{unknown }, \{\text{female only events}\})$.
    
    \noindent\textbf{Interaction with $E_3$}: \newline
    This is {unknown} to $E_3$ and she responds with
    $(\text{unknown }, \{\text{female only events}\})$.
    
    \
    \\
    \noindent Again the \emph{iterative strategy} collected the set of
    attributes known to follow and replies with
    $(\text{unknown }, \{\text{female only events}\})$.  The
    exploration algorithm introduces a fictitious counterexample for
    the attribute.

    \begin{figure}[t]
      \setlength{\belowcaptionskip}{-45pt}
      \begin{minipage}{0.49\textwidth}
        \centering \textbf{Counterexample Question 5} \vspace{1ex}
            
            \begin{cxt}
              \cxtName{$\context_{Q5}$} \cxtNichtKreuz{} \atr{$\geq$ 5
                events} \atr{$\geq$ 10 events} \atr{female only
                events} \atr{male only events} \atr{part of $\geq$ 8
                olympics} \obj{x.xx.}{Karate – Kumite}
              \obj{x.xx.}{Taekwondo}
            \end{cxt}
            
            \vspace{1ex}
          \end{minipage}
          \begin{minipage}{0.49\textwidth}
            \centering \textbf{Counterexample Question 8} \vspace{1ex}
            
            \begin{cxt}
              \cxtName{$\context_{Q8}$} \cxtNichtKreuz{} \atr{$\geq$ 5
                events} \atr{$\geq$ 10 events} \atr{female only
                events} \atr{male only events} \atr{part of $\geq$ 8
                olympics} \obj{..x.x}{Aq. – Artistic Swimming}
            \end{cxt}
            
            \vspace{1ex}
          \end{minipage}
          \caption{}
        \end{figure}
    
        \
        \\
        \noindent\textbf{Question 8}:
        $\{\text{female only events}\} \rightarrow \{\text{male only
          events}\}\ ?$
    
    \noindent\textbf{Interaction with $E_1$}: \newline
    This is {unknown} to $E_1$ and she responds with
    $(\text{unknown }, \{\text{male only events}\})$.
    
    \noindent\textbf{Interaction with $E_2$}: \newline
    The expert knows this to be false and responds with
    $(\text{false }, \context_{Q8})$.
    
    %
    %

    \noindent\textbf{Result of the collaborative attribute
      exploration:}
    
    The result of the exploration is the set of accepted implications
    $\mathcal{L}_{result}$ and the incomplete context of
    counterexamples $\context_{result}$ which contains the set of
    fictitious counterexamples
    \vspace{-1ex}
    \begin{align*}
      G^*=\{
      &g^?_{\{\geq\text{ 5 events},\text{male only events}, \text{ part of } \geq \text{ 8 olympics}, \geq \text{ 10 events}\}\not\rightarrow\{\text{female only events}\}},\\
      &g^?_{\{\geq \text{ 10 events}\}\not\rightarrow\{\text{female only events}\}},\\
      &g^?_{\{\geq \text{ 10 events}\}\not\rightarrow\{\text{male only events}\}}
        \}.
    \end{align*}
    %
    \vspace{-6ex}
    \begin{figure}
      \begin{minipage}{\textwidth}
        \centering \textbf{Result of the Attribute Exploration}
        \vspace{1ex}
            
            \begin{cxt}
              \cxtName{$\context_{result}$} \cxtNichtKreuz{}
              \atr{$\geq$ 5 events} \atr{$\geq$ 10 events} \atr{female
                only events} \atr{male only events} \atr{part of
                $\geq$ 8 olympics} \obj{..x.x}{Aquatics – Artistic
                Swimming} \obj{x.xxx}{Aquatics – Diving}
              \obj{..xx.}{Aquatics – Marathon Swimming}
              \obj{..xxx}{Aquatics – Water Polo} \obj{..xxx}{Cycling –
                Road} \obj{....x}{Equestrian – Dressage}
              \obj{....x}{Equestrian – Eventing}
              \obj{....x}{Equestrian – Jumping} \obj{..xxx}{Football}
              \obj{..xxx}{Hockey} \obj{x.xx.}{Karate – Kumite}
              \obj{..xxx}{Modern Pentathlon} \obj{..xx.}{Surfing}
              \obj{x.xx.}{Taekwondo} \obj{..xx.}{Triathlon}
              \obj{x..xx}{Wrestling – Greco Roman}
              \obj{xx.xx}{$g^?_{\{\geq\text{ 5 events},\text{male only
                    events}, \text{ part of } \geq \text{ 8 olympics},
                  \geq \text{ 10
                    events}\}\not\rightarrow\{\text{female only
                    events}\}}$}
              \obj{?x.??}{$g^?_{\{\geq \text{ 10
                    events}\}\not\rightarrow\{\text{female only
                    events}\}}$}
              \obj{?x?.?}{$g^?_{\{\geq \text{ 10
                    events}\}\not\rightarrow\{\text{male only
                    events}\}}$}
            \end{cxt}
            
            \vspace{1ex}
            \begin{align*}
              \mathcal{L}_{result}= \{&\{\geq \text{ 10 events}\} \rightarrow \{\geq \text{ 5 events}, \text{ part of } \geq \text{ 8 olympics}\},\\ &\{\geq \text{ 5 events}\} \rightarrow \{\text{male only events}\}\}
            \end{align*}
          \end{minipage}
          \caption{}
        \end{figure}

        \vspace{-9ex}
        We can see that each of the experts provided at least one
        example that none of the other experts could have provided to the
        context of counterexamples $\context_{result}$, namely \emph{Aquatics - Artistic Swimming},
        \emph{Cycling - Road} and \emph{Taekwondo}.
        Further the set $\mathcal{L}_{result}$ of accepted
        implications is larger than any of the individuals experts
        known implication sets. 
      \end{example}
    \end{document}